\documentclass[11pt]{article}

\usepackage[preprint]{acl}

\usepackage{times}
\usepackage{latexsym}
\usepackage[T1]{fontenc}
\usepackage[utf8]{inputenc}
\usepackage{microtype}
\usepackage{graphicx}
\usepackage{booktabs}
\usepackage{amsmath}
\usepackage{amssymb}
\usepackage{amsthm}
\newtheorem{theorem}{Theorem}

\newtheorem{corollary}{Corollary}
\theoremstyle{definition}

\theoremstyle{remark}

\usepackage{multirow}
\usepackage{array}
\usepackage{makecell}
\usepackage{xcolor}
\usepackage{colortbl}
\usepackage{arydshln}
\definecolor{ourrowbg}{HTML}{D8EEF2}
\definecolor{lineBase}{HTML}{888888}
\definecolor{lineTG}{HTML}{C0504D}
\definecolor{linePG}{HTML}{4F81BD}
\definecolor{lineOurs}{HTML}{1F8C9C}
\definecolor{ntTerm}{HTML}{777777}
\definecolor{ntAgent}{HTML}{2F6FB3}
\definecolor{ntUser}{HTML}{A87900}
\definecolor{ntTool}{HTML}{2F855A}
\definecolor{ntAuth}{HTML}{C04B4B}
\definecolor{ntDecision}{HTML}{7C4DAD}
\definecolor{ntSubflow}{HTML}{B7791F}
\usepackage{inconsolata}
\usepackage{pifont}
\usepackage{enumitem}
\usepackage{tikz}
\usetikzlibrary{arrows.meta,positioning,shapes.geometric,calc,fit,backgrounds}
\usepackage{pgfplots}
\pgfplotsset{compat=1.18}
\definecolor{lineBase}{HTML}{9098A1}
\definecolor{lineTG}{HTML}{E8912A}
\definecolor{linePG}{HTML}{5A6FC0}
\definecolor{lineOurs}{HTML}{12A19A}
\pgfplotsset{
  passk base/.style={
    axis lines=left,
    axis line style={draw=black, line width=0.6pt},
    tick align=outside, major tick length=2.5pt,
    xtick style={draw=black}, ytick style={draw=black},
    ymajorgrids=true, xmajorgrids=false,
    grid style={gray!20, line width=0.4pt, densely dashed},
    xticklabel style={font=\scriptsize}, yticklabel style={font=\scriptsize},
    title style={font=\small\bfseries, yshift=-0.3mm},
    label style={font=\scriptsize},
    ylabel near ticks, ylabel style={yshift=-1.5mm},
    clip=false,
    every axis plot/.append style={line width=1.1pt, mark size=1.7pt},
  },
  p direct/.style={lineBase, densely dashed, mark=*,
    mark options={solid, fill=white, draw=lineBase, line width=0.7pt}},
  p tg/.style={lineTG, mark=square*,
    mark options={fill=white, draw=lineTG, line width=0.7pt}},
  p pg/.style={linePG, mark=triangle*, mark size=2.1pt,
    mark options={fill=white, draw=linePG, line width=0.7pt}},
  p ours/.style={lineOurs, line width=1.7pt, mark=*,
    mark options={fill=lineOurs, draw=lineOurs}},
  passk legend/.style={legend style={font=\scriptsize, draw=none, fill=none,
    column sep=3mm, inner sep=1pt}, legend cell align=left},
}
\usepackage[ruled,vlined]{algorithm2e}
\usepackage{listings}
\usepackage[most]{tcolorbox}
\usepackage{dblfloatfix}

\newcommand{\taub}{$\tau$-bench}
\newcommand{\tautwo}{$\tau^2$-bench}
\newcommand{\pg}{\textsc{PolicyGuide}}
\newcommand{\pgs}{\textsc{PolicyGuide-Self}}
\newcommand{\pgd}{PolicyGuard}
\newcommand{\pgr}{\textsc{PolicyGuide-Raw}}
\newcommand{\tg}{ToolGuard}
\newcommand{\direct}{ReAct}

\newcommand{\passone}{$\textsc{Pass}^{1}$}
\newcommand{\passfour}{$\textsc{Pass}^{4}$}
\newcommand{\passk}{$\textsc{Pass}^{k}$}
\newcommand{\gptfive}{GPT~5.4}
\newcommand{\gptfourone}{GPT~4.1}
\newcommand{\claudesonnet}{Claude Sonnet~4.6}
\newcommand{\geminipro}{Gemini~2.5 Pro}
\newcommand{\modelicon}[2][0.9em]{\raisebox{-0.18\height}{\includegraphics[height=#1]{#2}}}
\newcommand{\oaiicon}[1][0.9em]{\modelicon[#1]{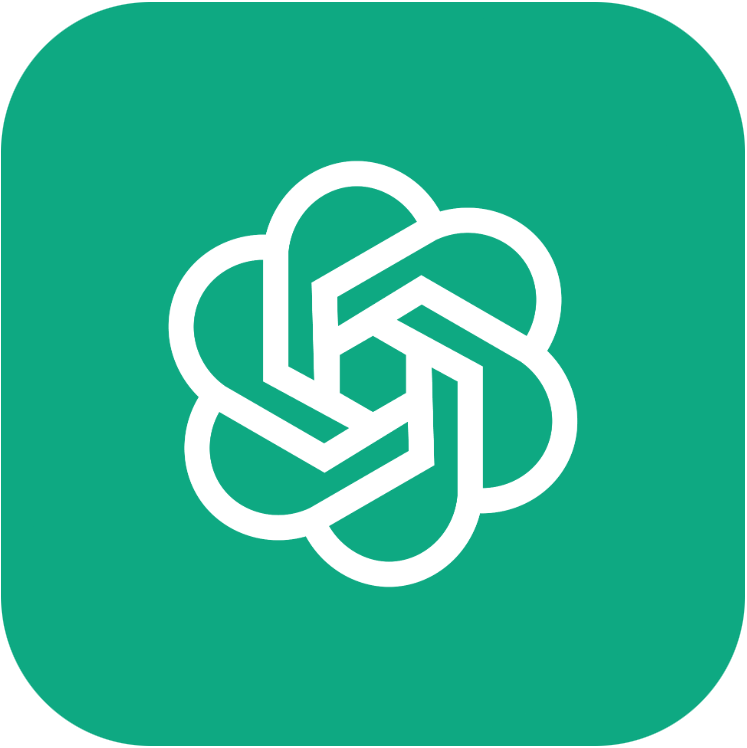}}
\newcommand{\claudeicon}[1][0.9em]{\modelicon[#1]{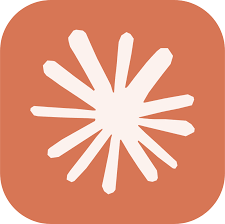}}
\newcommand{\geminiicon}[1][0.9em]{\modelicon[#1]{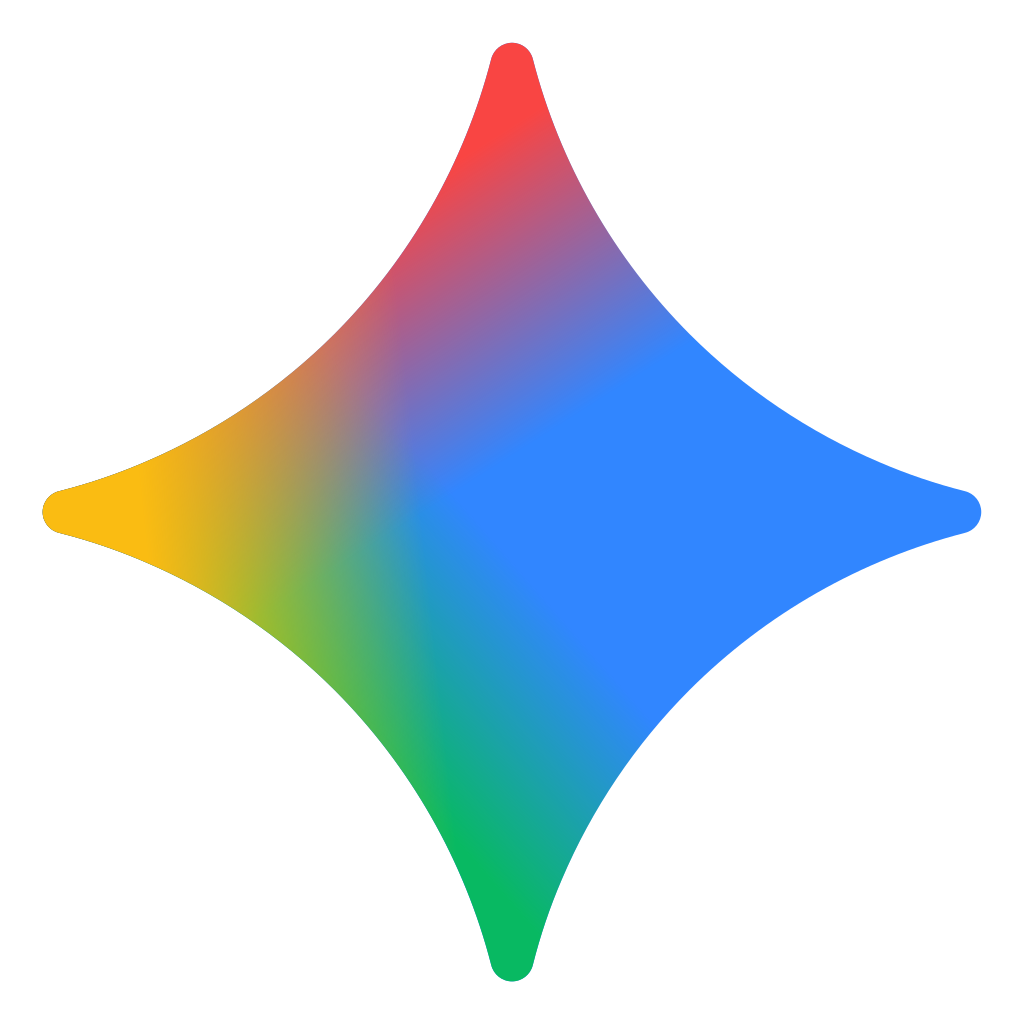}}
\newcommand{\tblgptfive}{\gptfive}
\newcommand{\tblclaudesonnet}{\claudesonnet}
\newcommand{\tblgeminipro}{\geminipro}
\newcommand{\airlineicon}[1][0.9em]{\modelicon[#1]{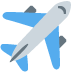}}
\newcommand{\retailicon}[1][0.9em]{\modelicon[#1]{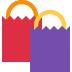}}
\newcommand{\telecomicon}[1][0.9em]{\modelicon[#1]{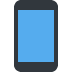}}
\newcommand{\airlinedomain}{Airline}
\newcommand{\retaildomain}{Retail}
\newcommand{\telecomdomain}{Telecom}
\newcommand{\tblairlinedomain}{\airlineicon[0.8em]\,\airlinedomain}
\newcommand{\tblretaildomain}{\retailicon[0.8em]\,\retaildomain}
\newcommand{\tbltelecomdomain}{\telecomicon[0.8em]\,\telecomdomain}
\newcommand{\nodetok}[2]{\tcbox[
  enhanced, nobeforeafter, on line, sharp corners,
  boxsep=0pt, left=2pt, right=2pt, top=0.5pt, bottom=0.5pt,
  boxrule=0pt, colback=#1!14, coltext=#1!45!black,
  tcbox raise base
]{\strut\ttfamily #2}}
\newcommand{\ntentry}{\nodetok{ntTerm}{entry/exit}}
\newcommand{\ntagent}{\nodetok{ntAgent}{agent\_action}}
\newcommand{\ntuser}{\nodetok{ntUser}{user\_input}}
\newcommand{\nttool}{\nodetok{ntTool}{tool\_call}}
\newcommand{\ntauth}{\nodetok{ntAuth}{tool\_authorization}}
\newcommand{\ntdecision}{\nodetok{ntDecision}{decision}}
\newcommand{\ntsubflow}{\nodetok{ntSubflow}{subflow}}
\newcommand{\circled}[1]{\textcircled{\raisebox{-0.4pt}{\scriptsize #1}}}

\title{\pg{}: From Guarding One Action to Guiding the Whole Workflow for Policy-Compliant LLM Agents}

\author{
  Seongjae Kang$^1$ \quad Taehyung Yu$^1$ \quad Sung Ju Hwang$^{1,2}$ \\
  $^1$KAIST \quad $^2$DeepAuto.ai \\
  \texttt{\{tjdwo2744, taehyung.yu, sjhwang\}@kaist.ac.kr}
}

\begin{document}
\maketitle

\begin{abstract}
Customer-service LLM agents must follow organizational policy when acting on a user's behalf.
Compliance failures arise from either forbidden actions, such as granting an ineligible change, or omitted procedural requirements, such as identification or confirmation.
Runtime safeguards can intervene on risky actions, but action-local checks do not guide an agent through a multi-step procedure.
Workflow-following systems support prescribed process execution, but primarily target workflow completion rather than safeguarding agent behavior.
\pg{} instead compiles each domain policy into a workflow graph and invokes a proactive verifier at user-turn boundaries.
From persisted graph state, the verifier reconciles open requests and returns step-specific remediation along a policy-compliant path.
Across the \tautwo{} airline, retail, and telecom domains with a \gptfive{} agent and verifier, \pg{} raises mean \passfour{} from $0.42$ to $0.62$, with the largest gain on telecom ($0.19$ to $0.61$), the most workflow-structured domain.
The same workflows transfer to \claudesonnet{} and \geminipro{} agents. Complementary evaluations find the lowest observed attack-success rate under adversarial users and the strongest procedural compliance in an independently validated workflow-level audit.
\end{abstract}

\section{Introduction}
\label{sec:intro}

\begin{figure}[t]
\centering
\definecolor{qblue}{HTML}{1976D2}
\definecolor{qgreen}{HTML}{2EAD4F}
\definecolor{qorange}{HTML}{F59E0B}
\definecolor{qpurple}{HTML}{B23AC7}
\definecolor{qpanel}{HTML}{F6F8FA}
\resizebox{\columnwidth}{!}{%
\begin{tikzpicture}[font=\sffamily,
    cell/.style={rounded corners=2pt, line width=0.65pt},
    ctitle/.style={font=\footnotesize\bfseries, text=black!88, align=center, inner sep=1pt},
    csys/.style={font=\tiny, text=black!78, align=center, inner sep=1pt},
    cdesc/.style={font=\scriptsize, text=black!88, align=center, inner sep=1pt},
    axis/.style={-{Stealth[length=1.8mm]}, line width=0.7pt, draw=black!65},
]
  \def\xa{0.55}\def\xm{4.30}\def\xb{4.45}\def\xc{8.20}
  \def\ya{0.75}\def\ym{3.05}\def\yb{3.20}\def\yc{5.50}
  \draw[cell, draw=black!32, fill=qpanel] (\xa,\ya) rectangle (\xm,\ym);
  \node[ctitle] at ({(\xa+\xm)/2},{(\ya+\ym)/2+0.62}) {Unguided agent};
  \node[csys]   at ({(\xa+\xm)/2},{(\ya+\ym)/2+0.12}) {Base agent};
  \node[cdesc]  at ({(\xa+\xm)/2},{(\ya+\ym)/2-0.52}) {Neither guides nor\\externally safeguards};
  \draw[cell, draw=qpurple!65!black, fill=qpurple!5] (\xb,\ya) rectangle (\xc,\ym);
  \node[ctitle] at ({(\xb+\xc)/2},{(\ya+\ym)/2+0.62}) {Action verifiers};
  \node[csys]   at ({(\xb+\xc)/2},{(\ya+\ym)/2+0.12}) {PolicyGuard, ToolGuard};
  \node[cdesc]  at ({(\xb+\xc)/2},{(\ya+\ym)/2-0.52}) {Externally check and gate\\the proposed action};
  \draw[cell, draw=qorange!75!black, fill=qorange!7] (\xa,\yb) rectangle (\xm,\yc);
  \node[ctitle] at ({(\xa+\xm)/2},{(\yb+\yc)/2+0.62}) {Workflow / SOP agents};
  \node[csys]   at ({(\xa+\xm)/2},{(\yb+\yc)/2+0.12}) {SOP-Agent, StateFlow, FlowAgent};
  \node[cdesc]  at ({(\xa+\xm)/2},{(\yb+\yc)/2-0.52}) {Follow a prescribed\\procedure};
  \draw[cell, draw=qgreen!75!black, fill=qgreen!10] (\xb,\yb) rectangle (\xc,\yc);
  \node[ctitle, text=qgreen!45!black]
    at ({(\xb+\xc)/2},{(\yb+\yc)/2+0.62}) {\pg{} (ours)};
  \node[csys] at ({(\xb+\xc)/2},{(\yb+\yc)/2+0.12}) {External workflow safeguard};
  \node[cdesc] at ({(\xb+\xc)/2},{(\yb+\yc)/2-0.52}) {Guides required steps\\and gates actions};
  \draw[axis] (\xa,0.42) -- (\xc,0.42);
  \node[anchor=north, font=\scriptsize, inner sep=2pt] at ({(\xa+\xc)/2},0.30)
    {\textbf{External safeguarding}\;---\;Monitors agent behavior};
  \draw[axis] (0.32,\ya) -- (0.32,\yc);
  \node[anchor=south, rotate=90, font=\scriptsize, inner sep=2pt] at (0.20,{(\ya+\yc)/2})
    {\textbf{Workflow enforcement}\;---\;Guides the required procedure};
\end{tikzpicture}%
}
\caption{Workflow systems execute prescribed procedures, while external
safeguards monitor agent behavior. \pg{} combines both roles.}
\label{fig:quadrant}
\vspace{-0.2in}
\end{figure}
\begin{figure*}[t]
\centering
\definecolor{ourblue}{HTML}{1976D2}
\definecolor{ourgreen}{HTML}{2EAD4F}
\definecolor{ourorange}{HTML}{F59E0B}
\definecolor{ourpurple}{HTML}{B23AC7}
\definecolor{ourred}{HTML}{D94848}
\definecolor{chipbg}{HTML}{F6F8FA}
\definecolor{chipbd}{HTML}{D4DADF}
\definecolor{dimfg}{HTML}{8B929A}
\resizebox{0.99\textwidth}{!}{%
\begin{tikzpicture}[
    font=\sffamily,
    bandhdr/.style={font=\small\bfseries, anchor=west},
    chip/.style={draw=chipbd, fill=chipbg, rounded corners=2pt, line width=0.4pt,
                 inner sep=3pt, font=\scriptsize, align=center, text width=20mm, minimum height=9mm},
    chipn/.style={chip, text width=13.5mm, inner sep=2.5pt},
    chipndim/.style={chipn, draw=ourred!45, fill=ourred!4, text=black!88},
    chipnwrite/.style={chipn, text width=15mm, inner sep=2pt, draw=ourred!75, fill=ourred!6,
                      line width=0.6pt, font=\scriptsize\ttfamily},
    chipwriteok/.style={chip, draw=ourgreen!80!black, fill=ourgreen!8, line width=0.6pt, font=\scriptsize\ttfamily},
    turnarrow/.style={->, >={Stealth[length=1.2mm,width=1.0mm]}, line width=0.65pt, draw=black!75},
    rail/.style={line width=0.8pt, draw=black!35},
    norail/.style={densely dashed, line width=0.6pt, draw=dimfg!70},
    asleep/.style={draw=dimfg, fill=white, circle, inner sep=1.6pt, line width=0.6pt},
    guardevt/.style={draw=ourpurple!80!black, fill=ourpurple!18, circle, inner sep=2.2pt, line width=0.8pt},
    guideevt/.style={draw=ourgreen!80!black, fill=ourgreen!22, circle, inner sep=2.2pt, line width=0.8pt},
    evtline/.style={densely dotted, line width=0.5pt},
    dirlab/.style={font=\scriptsize, text=black!88, align=center},
    guardlab/.style={font=\scriptsize, text=black!88, align=center},
    asleeplab/.style={font=\scriptsize, text=black!88, align=center},
    blockbadge/.style={draw=ourred, fill=ourred!10, rounded corners=3pt, inner xsep=4pt,
                       inner ysep=2.5pt, font=\scriptsize\bfseries, text=ourred!80!black, align=center},
    okbadge/.style={draw=ourgreen!80!black, fill=ourgreen!12, rounded corners=2pt, inner xsep=4pt,
                    inner ysep=2.5pt, font=\scriptsize\bfseries, text=ourgreen!40!black, align=center},
]
\node[bandhdr] at (0,0.4) {(a) Action guard:\ \ one \textsc{pass/block} at the final call};
\node[chipn]     (a1) at (0.9,-1.05) {\textbf{User:} change my flight to HAT228};
\node[chipndim]  (a2) at (2.75,-1.05) {\textbf{Agent:} skips identify, eligibility, confirmation};
\node[chipnwrite](a3) at (4.6,-1.05) {update\_\\res.\_flights\\(...)};
\node[chipn]     (a4) at (6.45,-1.05) {\textbf{Agent:} re-derives procedure from the error};
\node[blockbadge, text width=12mm] (afail) at (8.25,-1.05) {loops /\\wrong\\recovery};
\draw[turnarrow] (a1) -- (a2);
\draw[turnarrow] (a2) -- (a3);
\draw[turnarrow] (a3) -- (a4);
\draw[turnarrow] (a4) -- (afail);
\draw[rail] (0.1,-2.25) -- (9.1,-2.25);
\node[font=\scriptsize\bfseries, text=black!80, anchor=west] at (0.1,-1.98) {guard};
\node[asleep]   (sa1) at (2.75,-2.25) {};
\node[guardevt] (ga1) at (4.6,-2.25) {};
\draw[evtline, draw=dimfg]  (a2.south) -- (sa1);
\draw[evtline, draw=ourred] (a3.south) -- (ga1);
\node[asleeplab, anchor=north, text width=26mm] at (2.35,-2.45)
  {Inactive: skipped steps are not evaluated};
\node[guardlab, anchor=north, text width=36mm] at (6.35,-2.45)
  {\textbf{BLOCK} at the final call: unmet prerequisites are reported only after the procedure has failed};
\draw[line width=0.5pt, draw=chipbd] (9.55,0.55) -- (9.55,-3.65);
\node[bandhdr] at (10.0,0.4) {(b) Workflow / SOP agent:\ \ workflow execution is part of the agent};
\node[chipn]     (w1) at (10.9,-1.05) {\textbf{User:} change my flight to HAT228};
\node[chipn]     (w2) at (12.75,-1.05) {\textbf{Agent:} follows the prescribed workflow};
\node[chipn]     (w3) at (14.6,-1.05) {\textbf{Agent:} completes the required checks};
\node[chipnwrite](w4) at (16.45,-1.05) {update\_\\res.\_flights\\(...)};
\node[okbadge, text width=12mm] (wfail) at (18.25,-1.05) {procedure\\completed};
\draw[turnarrow] (w1) -- (w2);
\draw[turnarrow] (w2) -- (w3);
\draw[turnarrow] (w3) -- (w4);
\draw[turnarrow] (w4) -- (wfail);
\draw[norail] (10.0,-2.25) -- (19.1,-2.25);
\node[font=\scriptsize\bfseries, text=black!80, anchor=west] at (10.0,-1.98) {(agent-integrated)};
\node[asleeplab, anchor=north, text width=34mm] at (11.95,-2.45)
  {Workflow control is implemented within the agent runtime};
\node[guardlab, anchor=north, text width=38mm] at (16.45,-2.45)
  {Its checks support faithful workflow execution within that runtime};
\node[bandhdr, text=ourgreen!40!black] at (0,-3.95) {(c) \pg{} (ours):\ \ guide generation at user-turn boundaries, walking the policy graph};
\node[chip]      (b1) at (1.15,-5.10) {\textbf{User:} change my flight to HAT228};
\node[chip]      (b2) at (4.0,-5.10) {\textbf{Agent:} asks for the user id};
\node[chip]      (b3) at (6.85,-5.10) {\textbf{Agent:} \texttt{get\_user\_} \texttt{details} \ding{51}};
\node[chip]      (b4) at (9.7,-5.10) {\textbf{Agent:} checks eligibility, presents summary};
\node[chip]      (b5) at (12.55,-5.10) {\textbf{User:} confirms ``yes''};
\node[chipwriteok](b6) at (15.4,-5.10) {update\_res.\_\\flights(...) \ding{51}};
\node[okbadge, text width=16mm] (bok) at (18.05,-5.10) {task\\completed\\in policy};
\draw[turnarrow] (b1) -- (b2);
\draw[turnarrow] (b2) -- (b3);
\draw[turnarrow] (b3) -- (b4);
\draw[turnarrow] (b4) -- (b5);
\draw[turnarrow] (b5) -- (b6);
\draw[turnarrow] (b6) -- (bok);
\draw[rail, draw=ourgreen!55] (0.1,-6.15) -- (19.1,-6.15);
\node[font=\scriptsize\bfseries, text=ourgreen!35!black, anchor=west] at (0.1,-5.88) {guide (user turns; workflow safeguards)};
\node[guideevt] (ge1) at (2.6,-6.15) {};
\node[guideevt] (ge2) at (5.45,-6.15) {};
\node[guideevt] (ge3) at (14.0,-6.15) {};
\draw[evtline, draw=ourgreen!80!black] (ge1) -- (b2.south west);
\draw[evtline, draw=ourgreen!80!black] (ge2) -- (b3.south west);
\draw[evtline, draw=ourgreen!80!black] (ge3) -- (b6.south west);
\node[dirlab, anchor=north, text width=25mm] at (2.6,-6.35)
  {\textbf{\ding{172}} User turn: identify the user first};
\node[dirlab, anchor=north, text width=44mm] at (6.45,-6.35)
  {\textbf{\ding{173}} User turn: load the record, verify eligibility, present the change, and request confirmation};
\node[dirlab, anchor=north, text width=27mm] at (14.0,-6.35)
  {\textbf{\ding{174}} User confirms: prerequisites met; \textbf{mutation recommended} with grounded arguments};
\end{tikzpicture}%
}
\caption{One task under three enforcement regimes. \textbf{(a)} An action guard
checks only the final mutating call, so skipped procedure is discovered late and
returned as a block. \textbf{(b)} A workflow/SOP agent drives the procedure, but
its checks are designed for faithful workflow execution rather than as a
safeguard against policy-violating behavior by a general-purpose agent. \textbf{(c)} \pg{} runs as an external, advisory guide:
at user-turn boundaries it tracks graph position across turns and stops at the first
unsatisfied node. If the agent attempts a mutating tool call before completing the
workflow, the runtime returns remediation for the unmet step. It recommends the
mutation only after the required workflow steps are grounded.}
\label{fig:concept}
\vspace{-0.1in}
\end{figure*}

LLM agents are beginning to support customer-service work, including booking flights, modifying orders, and changing account plans through tools on user accounts.
These systems typically pair a general-purpose reasoning-and-acting loop~\citep{yao2023react} with a frontier model such as \mbox{\oaiicon\,\tblgptfive{}}, \mbox{\claudeicon\,\tblclaudesonnet{}}, or \mbox{\geminiicon\,\tblgeminipro{}}~\citep{openai2026gpt54,anthropic2026sonnet46,comanici2025gemini25}.
Because these models are large and closed-weight, domain-specific fine-tuning is often unavailable or impractical; runtime safeguards offer an integration point that does not require retraining.
\taub{}~\citep{yao2024taubench} and \tautwo{}~\citep{barres2025tau2} evaluate this setting against natural-language policies in \mbox{\airlineicon\,\airlinedomain{}}, \mbox{\retailicon\,\retaildomain{}}, and \mbox{\telecomicon\,\telecomdomain{}}.

Policy compliance depends on both the selected action and the procedure used to reach it.
An agent may grant an ineligible change, or it may skip or misorder identification, eligibility checks, and confirmation.
Such procedural failures can produce a forbidden outcome or leave an otherwise permissible action unsupported.
Our source-policy analysis (Appendices~\ref{app:policy} and~\ref{app:policy:axis2}) finds that procedural requirements are pervasive ($67.4\%$ in airline, ${\sim}100\%$ in retail, and $98.0\%$ in telecom), whereas ordered workflow requirements concentrate in telecom ($54.0\%$, versus $4.7\%$ in airline and $3.6\%$ in retail).
Flat prerequisites can often be checked when the agent proposes a guarded
action, such as a mutating tool call. Ordered requirements also constrain
earlier dialogue and tool-use actions.
For example, telecom troubleshooting follows diagnose--instruct--verify sequences that may contain no agent-side mutation for an action guard to intercept.

These two failure modes motivate complementary capabilities (Figure~\ref{fig:quadrant}).
\emph{Safeguarding} monitors agent behavior and intervenes on risky actions~\citep{zwerdling2025toolguard,chen2025shieldagent,xiang2025guardagent}, whereas \emph{workflow enforcement} steers execution through required steps~\citep{ye2025sopagent,wu2024stateflow,shi2025flowagent}.
The two literatures thus target different primary objectives: safe agent behavior and faithful workflow completion.
\pgd{}~\citep{kang2026policyguard} provides the closest connection by incorporating procedural remediation into a mutating-call safeguard, but remains action-triggered and cannot cover earlier deviations outside its guarded action class.

We propose \pg{}, an external runtime guide for policy-compliant agents (Figure~\ref{fig:concept}).
\pg{} compiles each domain policy into a workflow graph.
At user-turn boundaries, a proactive verifier traverses the graph from its persisted position, reconciles all open requests, and returns focused remediation for the first unmet step.
Persisted state lets the verifier apply workflow safeguards throughout the interaction while coordinating multiple requests.
The result is an agent-agnostic external overlay that pairs the same procedural representation with different agents.
Across the \tautwo{} airline, retail, and telecom domains with a \gptfive{} agent and verifier, \pg{} raises mean \passfour{} across domains from $0.42$ (unguided) to $0.62$ (\S\ref{sec:results}), with the largest gain on telecom ($0.19$ to $0.61$), the domain whose policy is most workflow-structured.
We additionally evaluate diagnostic workflow variants and a matched FlowAgent
workflow-controller baseline on telecom
(\S\ref{sec:ablation}--\ref{sec:flowagent}).
The same workflows transfer to \claudesonnet{} and \geminipro{} agents
(\S\ref{sec:modelgen}). We further evaluate adversarial robustness with
CRAFT red-teaming and workflow compliance with an independently validated Telecom
trace audit (\S\ref{sec:robust} and~\S\ref{sec:trace-compliance}).

\paragraph{Contributions.}
\begin{itemize}[itemsep=0pt,topsep=2pt,leftmargin=*]
\item We characterize policy compliance as a joint safeguarding and
workflow-enforcement problem: agents must avoid impermissible actions while
completing required procedural steps, including those occurring before or
outside a guarded action class.
\item We introduce \pg{}, an external proactive verifier that compiles policies into workflow graphs, tracks multiple open requests across turns, and guides the agent through the required steps.
\item We demonstrate 20-point mean \passfour{} gains and cross-agent transfer.
Complementary analyses show robustness to CRAFT red-team attacks and
stronger workflow compliance.
\end{itemize}

\section{Background and Related Work}
\label{sec:related}

\pg{} connects runtime safeguards, which monitor agent behavior, with workflow-guided systems, which execute prescribed procedures.
It gives persisted workflow state a safeguarding role over agent behavior rather than making workflow control the agent architecture; separating verifier from actor additionally enables reuse across agent runtimes (Figure~\ref{fig:quadrant}).

\subsection{\tautwo{} and policy-adherent agents}
\tautwo{}~\citep{barres2025tau2}, building on \taub{}~\citep{yao2024taubench},
benchmarks policy-adherent LLM agents in customer-service domains with a
natural-language policy, read-only tools, and mutating tools. We use the airline,
retail, and telecom domains: each task is either policy-violation (the agent must
refuse) or mutation (the agent must act correctly), and success requires both the
final database state and the natural-language assertions to hold. Telecom adds
dual control, where some required actions are user-side tools the agent cannot
call, making workflow order especially visible. Nearby benchmarks target
complementary questions: CRMArena-Pro studies confidentiality compliance rather
than ordered procedures~\citep{huang2025crmarena}; IntellAgent generates
diagnostic tests from policy graphs~\citep{levi2025intellagent}; Near-Miss audits
failures post hoc~\citep{rabinovich2026nearmiss}; AgentRewardBench evaluates
trajectory judges~\citep{lu2025agentrewardbench}; and CRAFT supplies adversarial
users rather than a benign workflow-completion benchmark~\citep{nakash2025craft}.

\input{figs/genpipeline_v2}

\subsection{Runtime safeguards: action verification}
\label{sec:related:guards}
Runtime safeguards commonly verify proposed tool calls. \tg{} compiles
tool-level guards~\citep{zwerdling2025toolguard}, while Solver-Aided checks
call constraints with solver support~\citep{winston2026solver}. Both can use
conversation history, but their intervention begins when a guarded call is
proposed. \pg{} instead maintains request-specific workflow state across turns
and proactively remediates unmet steps. PCAS monitors event traces
\citep{palumbo2025pcas}, while ShieldAgent wraps agents with structural safety
checks~\citep{chen2025shieldagent}; both track order but reduce dialogue
semantics to predicates or keyword matching. GuardAgent is a single-turn
admission controller~\citep{xiang2025guardagent}, ToolSafe classifies unsafe
tool use~\citep{mou2026toolsafe}, and AgentSpec specifies tool-agent safety
properties~\citep{wang2026agentspec}; these target broader tool-risk settings
rather than multi-turn business procedures. AGrail adapts checks online
\citep{luo2025agrail}, Conseca synthesizes just-in-time policies
\citep{tsai2025conseca}, and Progent enforces least privilege over tool
arguments~\citep{shi2025progent}, but still act locally around risky actions.
\pgd{}~\citep{kang2026policyguard} is the closest action-guard baseline: it reads
the full conversation, checks a mutating call against a per-tool checklist, and
returns \textsc{pass}/\textsc{block} with remediation. This makes it much more
dialogue-aware than argument-only guards, but it remains action-scoped: it does
not persist workflow position or proactively guide the agent through
the missing steps before a mutating call is attempted.

\subsection{Workflow- and SOP-guided agents}
\label{sec:related:workflow}
A parallel line encodes procedures as traversable graphs or state machines.
SOP-Agent~\citep{ye2025sopagent} compiles a standard operating procedure into a
decision graph that restricts actions at each node. StateFlow
\citep{wu2024stateflow} represents a task as a finite-state machine whose states
hold prompts and tool calls. SMoT maintains explicit task state
\citep{liu2023smot}; MetaGPT and ProAgent organize agents around procedural
roles or plans~\citep{hong2024metagpt,ye2023proagent}. JourneyBench studies
dynamic prompting in our domain~\citep{balaji2026journeybench}; FLAP enforces
flows through constrained decoding~\citep{roy2024flap}; and FlowBench finds that
even strong models struggle to follow supplied workflows reliably
\citep{xiao2024flowbench}.

These systems primarily study faithful workflow execution rather than
safeguarding against policy-violating agent behavior. FlowAgent is the closest
qualification~\citep{shi2025flowagent}: its pre- and post-decision controllers
guide execution and can reject invalid transitions. Its focus, however, is
compliant and flexible workflow execution under out-of-workflow requests; it is
not framed or evaluated as a safeguard against policy-violating agent behavior.
\pg{} instead gives persisted workflow state a safeguarding role: a separate
verifier monitors the interaction, returns remediation for unmet steps, and is
evaluated with both benign and manipulative users. Verifier separation also
enables workflow reuse across agents, a practical benefit.

\section{Method}
\label{sec:method}

\pg{} guides a general-purpose LLM agent with an offline policy workflow and an external runtime verifier; code persists workflow state and delivers next-step remediation (Figure~\ref{fig:method}).
Conceptually, \pg{} is a \emph{reference-monitor-inspired runtime safeguard}~\citep{anderson1972computer,schneider2000enforceable}: it observes the interaction and can intervene before policy-sensitive actions, but the evaluated configuration steers execution rather than claiming classical mandatory enforcement.
Its repeated judgments over the growing trace relate to runtime verification~\citep{leucker2009brief,bauer2011runtime}, and its request state to dialogue-state tracking~\citep{williams2013dstc,henderson2014word}.

\subsection{Theoretical motivation}
\label{sec:theory}
An action-triggered verifier mediates only the actions that invoke it, such as
proposed mutating tool calls. This is sufficient only when every reachable first
deviation occurs at such an action. Policy workflows, however, can constrain
other agent actions, including evidence gathering, user-facing instructions,
branch selection, and completion decisions. These deviations matter even when
the eventual mutation is permissible, because a later action check cannot undo
an already-committed procedural violation.
Appendix~\ref{app:theory} formalizes this distinction through
\emph{intervention coverage}. Theorem~\ref{thm:complete-mediation} shows that an
ideal binding verifier preserves procedural validity exactly when its firing
schedule covers every reachable first deviation.
Corollary~\ref{cor:workflow-vs-action} shows that an ideal workflow-level
schedule satisfies this condition, whereas an action-triggered schedule does so
only when every first deviation itself triggers the check.

\subsection{Policy workflow representation}
\label{sec:graph}
A workflow is \textbf{the graph of the policy-compliant interaction}, with
nodes for actors and actions and edges for transitions. All three domains
start with shared intake, identification, and classification, then enter
request-specific subflows that reuse shared procedures and express decision
gates or diagnostic chains.

Node types are \ntentry{} (structure), \ntagent{} (non-tool agent action),
\ntuser{} (user response), \nttool{} (read-only tool), \ntauth{} (mutating-tool
authorization), \ntdecision{} (branch), and \ntsubflow{} (subflow invocation).
Each node specification names its actor and expected action and states an
explicit satisfying condition that the runtime verifier judges against the
interaction (\S\ref{sec:guide}); subflows are inlined at load time, so the
runtime traverses one flat graph per domain. A mutating tool call is enabled at
its authorization node and verified from the corresponding tool result.

\begin{algorithm}[!t]
\small
\DontPrintSemicolon
\SetKwInOut{Input}{Input}
\Input{\hspace{0.2em}policy $\pi$, tools $\mathcal{T}$, workflow $G$, history $H$, and state $S$}
\BlankLine
$\widehat{\mathcal{R}} \leftarrow$ requests in $H$ reconciled with the tracked requests in $S$\;
\ForEach{open request $r\in\widehat{\mathcal{R}}$}{
  $p \leftarrow$ entry of $G$ if $r$ is new; otherwise its position in $S$\;
  \While{$p$ is nonterminal}{
    \If{the requirement at $p$ is not satisfied by $H$}{break\;}
    $p \leftarrow$ successor along the outgoing edge in $G$ that matches $H$\;
  }
  \eIf{$p$ is terminal}{
    $d_r \leftarrow \varnothing$\;
  }{
    $d_r \leftarrow$ action required at $p$\;
  }
}
$d \leftarrow$ merge $\{d_r:r\in\widehat{\mathcal{R}}\}$\;
return $(d,\widehat{\mathcal{R}})$\;
\caption{\pg{} verifier}
\label{alg:guide}
\end{algorithm}

\subsection{Offline workflow generation}
\label{sec:gen}
The workflows are generated offline by a multi-stage pipeline and frozen once
per domain for all workflow-based conditions (Figure~\ref{fig:method}, top).
\textbf{Stage 1} extracts tool specifications and mutating tools, excluding
user-device actions.  \textbf{Stage 2} derives request types, shared procedures,
ordered subflows, and a coverage audit; \textbf{Stage 3} reviews the plan.
\textbf{Stage 4} generates and schema-validates subflows (one repair retry and
branch review), and \textbf{Stage 5} connects the intake spine, classifier, and
subflows. \textbf{Stage 6} validates schema conformance, tool inventory,
mutating-tool authorization coverage, graph composition, edge arity, and
reachability; reviews policy-to-graph mappings; and prunes unused subflows.
The authors manually verified request coverage, prerequisite order, and
mutating-action authorization and verification against the source policy and
tool specifications, without editing the generated workflows.
Appendices~\ref{app:prompts}, \ref{app:workflow-example}, and
\ref{app:workflow-validation} provide the prompts, examples, and validation
checks and semantic-audit scope, respectively.

\begin{table*}[!t]
\centering
\small
\setlength{\tabcolsep}{6pt}
\renewcommand{\arraystretch}{1.1}
\resizebox{\textwidth}{!}{%
\begin{tabular}{lll ccc c ccc c ccc}
\toprule
& & & \multicolumn{3}{c}{\textbf{\airlinedomain} (50)} && \multicolumn{3}{c}{\textbf{\retaildomain} (114)} && \multicolumn{3}{c}{\textbf{\telecomdomain} (114)} \\
\cmidrule(lr){4-6}\cmidrule(lr){8-10}\cmidrule(lr){12-14}
& System & Verifier & Overall & PV & Mut && Overall & PV & Mut && Overall & PV & Mut \\
\midrule
\multirow{4}{*}{\rotatebox{90}{\passone{}}}
 & \direct{}     & ---         & 0.640 & 0.865 & 0.433 && 0.800 & 0.929 & 0.792 && 0.384 & 0.721 & 0.180 \\
 & \tg{}         & static code & 0.575 & 0.969 & 0.212 && ---   & ---   & ---   && ---   & ---   & ---   \\
 & \pgd{}        & \tblgptfive{} & 0.710 & \textbf{1.000} & 0.442 && 0.800 & \textbf{0.964} & 0.790 && 0.406 & 0.733 & 0.208 \\
\rowcolor{ourrowbg}
 \cellcolor{white} & \textbf{\pg{}} & \tblgptfive{} & \textbf{0.775} & 0.979 & \textbf{0.587} && \textbf{0.809} & \textbf{0.964} & \textbf{0.799} && \textbf{0.866} & \textbf{0.895} & \textbf{0.849} \\
\midrule
\multirow{4}{*}{\rotatebox{90}{\passfour{}}}
 & \direct{}     & ---         & 0.460 & 0.750 & 0.192 && 0.596 & \textbf{0.857} & 0.579 && 0.193 & 0.442 & 0.042 \\
 & \tg{}         & static code & 0.520 & 0.875 & 0.192 && ---   & ---   & ---   && ---   & ---   & ---   \\
 & \pgd{}        & \tblgptfive{} & 0.580 & \textbf{1.000} & 0.192 && 0.605 & \textbf{0.857} & 0.589 && 0.202 & 0.488 & 0.028 \\
\rowcolor{ourrowbg}
 \cellcolor{white} & \textbf{\pg{}} & \tblgptfive{} & \textbf{0.620} & 0.917 & \textbf{0.346} && \textbf{0.614} & \textbf{0.857} & \textbf{0.598} && \textbf{0.614} & \textbf{0.721} & \textbf{0.549} \\
\bottomrule
\end{tabular}%
}
\caption{Main results on the base splits (\tblgptfive{} agent, $n{=}4$; airline 50,
retail/telecom 114 tasks). Cells report Pass$^1$ and Pass$^4$ overall and on the
PV/Mut slices. The verifier is absent for \direct{}, static code for \tg{}, and
\gptfive{} for \pgd{} and \pg{}. Retail \pgd{} uses the measured one-turn-scope fix.}
\label{tab:main}
\end{table*}

\begin{figure*}[!t]
\centering
\pgfplotsset{
  passk axis/.style={
    passk base,
    width=0.31\linewidth, height=4.4cm,
    xmin=0.85, xmax=4.15, ymin=0.12, ymax=0.94,
    xtick={1,2,3,4}, xticklabels={$\textsc{P}^{1}$,$\textsc{P}^{2}$,$\textsc{P}^{3}$,$\textsc{P}^{4}$},
    ytick={0.2,0.3,0.4,0.5,0.6,0.7,0.8,0.9},
  },
}
\begin{tikzpicture}
\begin{axis}[passk axis, title={Airline}, ylabel={\passk{}}, ylabel near ticks,
    name=ax1, legend to name=passklegend, legend columns=4, passk legend]
\addplot[p direct] coordinates {(1,0.640)(2,0.530)(3,0.485)(4,0.460)};
\addplot[p tg] coordinates {(1,0.575)(2,0.553)(3,0.535)(4,0.520)};
\addplot[p pg] coordinates {(1,0.710)(2,0.630)(3,0.595)(4,0.580)};
\addplot[p ours] coordinates {(1,0.775)(2,0.707)(3,0.660)(4,0.620)};
\addlegendimage{p direct}\addlegendentry{\direct{}}
\addlegendimage{p tg}\addlegendentry{\tg{}}
\addlegendimage{p pg}\addlegendentry{\pgd{}}
\addlegendimage{p ours}\addlegendentry{\pg{}}
\end{axis}
\begin{axis}[passk axis, title={Retail}, at={(ax1.south east)}, xshift=4mm, anchor=south west, yticklabels={,,}, name=ax2]
\addplot[p direct] coordinates {(1,0.800)(2,0.700)(3,0.638)(4,0.596)};
\addplot[p pg] coordinates {(1,0.800)(2,0.703)(3,0.643)(4,0.605)};
\addplot[p ours] coordinates {(1,0.809)(2,0.715)(3,0.654)(4,0.614)};
\end{axis}
\begin{axis}[passk axis, title={Telecom}, at={(ax2.south east)}, xshift=4mm, anchor=south west, yticklabels={,,}, name=ax3]
\addplot[p direct] coordinates {(1,0.384)(2,0.273)(3,0.226)(4,0.193)};
\addplot[p pg] coordinates {(1,0.406)(2,0.292)(3,0.237)(4,0.202)};
\addplot[p ours] coordinates {(1,0.866)(2,0.763)(3,0.682)(4,0.614)};
\end{axis}
\end{tikzpicture}
\vspace{0.2em}\centerline{\ref{passklegend}}
\caption{Pass$^k$ vs.\ $k$ (reliability; higher is better) on the \textbf{base}
split of each domain, \tblgptfive{}, all cells $n{=}4$.}
\label{fig:main}
\end{figure*}

\subsection{Online policy-guided runtime}
\label{sec:guide}
Algorithm~\ref{alg:guide} summarizes the runtime (Figure~\ref{fig:method},
bottom).  Each firing is a single verifier generation
\[
  V_\phi(\pi,\mathcal{T},G,H,S)=(d,\widehat{\mathcal{R}}),
\]
where $V_\phi$ is the verifier LLM; $\pi$, $\mathcal{T}$, $G$, $H$, and $S$
are the raw policy, tool specifications, frozen workflow graph, interaction
history, and code-owned request state; and $d$ and $\widehat{\mathcal{R}}$ are
the merged remediation and the updated request records the runtime persists.

\paragraph{Firing and interface.}
The verifier fires before the agent responds to each user turn, and once more
after a mutating tool call not authorized by the current workflow state is
intercepted; skipped tool-result turns are folded into the next firing's
conversation delta, so each call judges the complete trajectory. Its prompt is
a cached static prefix ($\pi$,
$\mathcal{T}$, the rendered graph, judging rules, and output contract) plus
the conversation and the latest state record; it returns a free-text audit and
one structured record per open request---node walk with cited evidence,
position, status, mutating-tool authorization, selection memory, and
remediation---plus a global transfer flag (Appendices~\ref{app:cost:cache}
and~\ref{app:prompts}).  It runs at temperature~0, model-paired with the agent.

\paragraph{Reconcile and traverse.}
The verifier reconciles the open requests against $S$ (continuing requests
keep their recorded positions; new ones open at the graph entry; abandoned or
duplicate ones are dropped or merged), then walks each from its recorded node,
judging every node's satisfying condition against $H$: facts and eligibility
count only when confirmed by tool results, while the user's own choices and
consent count from their messages.  The walk stops at the first unsatisfied
node, whose required action becomes the remediation; one generation can advance
several nodes, and terminal nodes mark a request done.

\paragraph{State and delivery.}
Code, rather than the model's conversational memory, owns state persistence. It
rejects unknown node IDs, filters authorization outputs against the enumerated
mutating-tool inventory, reconstructs the currently enabled tool set, and
persists each request's position and memory. The merged remediation is injected
as a guidance message before the agent acts.

\paragraph{Intervention.}
In the evaluated advisory mode, the first mutating tool call not authorized by
the current workflow state within each user-turn region is removed before
execution and triggers a corrective verifier firing. The gate then disarms
until the next user turn, allowing remediation without repeated interception.
This preserves the execution loop during recovery: intervention corrects the
actor without authorizing or replaying the rejected call. An actor that
deliberately ignores remediation and immediately repeats an unauthorized write
can circumvent this advisory gate. Other workflow-governed actions are guided
by the same remediation channel.
Appendix~\ref{app:trace-validation} reports trajectory-based stress tests and
a recorded authorization refresh.

\subsection{Variants and ablations}
\label{sec:variants}
\pg{} rests on two separable ingredients: \emph{what} the policy is compiled
into (the graph versus the raw policy text) and \emph{who} tracks progress
(an external verifier versus the acting agent).  Each variant strips one.
\pgr{} keeps the verifier model, firing schedule, carried state, and remediation
channel but substitutes the raw policy for $G$, so no graph position
persists---isolating the compiled graph.  \pgs{} places the frozen graph
in the actor's system prompt but removes the external verifier, code-owned
state, per-turn remediation, and corrective intercept---isolating external
tracking (\S\ref{sec:ablation}).

\section{Experiments}
\label{sec:experiments}

\subsection{Setup}
\label{sec:setup}

We evaluate the \tautwo{} base splits: \airlineicon\,Airline (50 tasks; 24 PV / 26 Mut),
\retailicon\,Retail (114; 7 PV / 107 Mut), and \telecomicon\,Telecom (114; 43
PV / 71 Mut). Retail main cells use \pgd{}'s taxonomy. PV tasks require
preventing forbidden mutations; Mut tasks require completing permitted
mutations with policy prerequisites. Diagnostic variants and FlowAgent use the
fixed benchmark test splits from \texttt{split\_tasks.json}: Retail (40; 4 PV /
36 Mut) and Telecom (40; 21 PV / 19 Mut). Telecom test is more PV-heavy than
base (52.5\% versus 37.7\%), so we report both slices. Actors and verifiers are
paired within model families:
\oaiicon\ \gptfive{}, \claudeicon\ \claudesonnet{}, and
\geminiicon\ \geminipro{}; the frozen user simulator is \gptfourone{}.
The domain-wide comparison uses \tblgptfive{}, which authored one frozen
workflow per domain, reused across systems and agent families to isolate
runtime and executor differences from re-authoring.

The main comparison uses the same task IDs and \tblgptfive{} substrate for
\direct{} (no guard), \pgd{}, and \pg{}, with released \tg{} code guards on
Airline. All main cells use $n{=}4$. We report Pass$^1$ and Pass$^4$;
PV and Mut denote slice-level Pass$^4$ unless labeled Pass$^1$.
For a task with $c$ successful trials,
$\text{Pass}^k=\binom{c}{k}/\binom{n}{k}$, averaged across tasks.
We also include FlowAgent~\citep{shi2025flowagent} as a matched
workflow-controller baseline on Telecom (\S\ref{sec:flowagent}).
The standard evaluators check final database state and natural-language
assertions. To measure intermediate procedural order, we supplement Pass$^k$
with an independently validated Telecom event-order rubric (\S\ref{sec:trace-compliance})
and report the guard-derived Call-NMR audit separately
(Appendix~\ref{app:nearmiss}).


\begin{table}[t]
\centering
\small
\setlength{\tabcolsep}{3.5pt}
\resizebox{\columnwidth}{!}{%
\begin{tabular}{llcccc}
\toprule
Domain & Metric & \direct{} & \makecell{\textsc{PolicyGuide}\\[-1pt]\textsc{Self}} & \makecell{\textsc{PolicyGuide}\\[-1pt]\textsc{Raw}} & \cellcolor{ourrowbg}\textbf{\pg{}} \\
\midrule
\multirow{3}{*}{\airlinedomain}
 & Overall & 0.460 & 0.480 & 0.520 & \cellcolor{ourrowbg}\textbf{0.620} \\
 & PV      & 0.750 & 0.833 & 0.875 & \cellcolor{ourrowbg}\textbf{0.917} \\
 & Mut     & 0.192 & 0.154 & 0.192 & \cellcolor{ourrowbg}\textbf{0.346} \\
\midrule
\multirow{3}{*}{\retaildomain}
 & Overall & 0.575 & 0.350 & 0.575 & \cellcolor{ourrowbg}\textbf{0.725} \\
 & PV      & 0.750 & 0.750 & 0.750 & \cellcolor{ourrowbg}\textbf{1.000} \\
 & Mut     & 0.556 & 0.306 & 0.556 & \cellcolor{ourrowbg}\textbf{0.694} \\
\midrule
\multirow{3}{*}{\telecomdomain}
 & Overall & 0.250 & 0.325 & 0.350 & \cellcolor{ourrowbg}\textbf{0.675} \\
 & PV      & 0.429 & 0.571 & 0.619 & \cellcolor{ourrowbg}\textbf{0.667} \\
 & Mut     & 0.053 & 0.053 & 0.053 & \cellcolor{ourrowbg}\textbf{0.684} \\
\bottomrule
\end{tabular}%
}
\caption{Workflow ablations (\gptfive{} agent; Airline base split, Retail and
Telecom benchmark test splits of 40 tasks). All cells report Pass$^4$.}
\label{tab:ablation}
\end{table}

\begin{table}[t]
\centering
\small
\setlength{\tabcolsep}{5pt}
\begin{tabular}{@{}llc@{}}
\toprule
System & Runtime control & Pass$^4$ \\
\midrule
\direct{} & actor only & 0.250 \\
\pgd{} & action-local check & 0.325 \\
FlowAgent & PDL + API control & 0.350 \\
\rowcolor{ourrowbg}
\textbf{\pg{}} & external graph verifier & \textbf{0.675} \\
\bottomrule
\end{tabular}
\caption{Matched workflow-controller comparison on the 40-task Telecom
benchmark test split. All values are Pass$^4$.}
\label{tab:flowagent}
\end{table}

\begin{table}[t]
\centering
\small
\setlength{\tabcolsep}{4pt}
\resizebox{\columnwidth}{!}{%
\begin{tabular}{llccc}
\toprule
Agent & Metric & \direct{} & \pgd{} & \cellcolor{ourrowbg}\textbf{\pg{}} \\
\midrule
\multirow{3}{*}{\tblgptfive{}}
 & Overall & 0.460 & 0.580 & \cellcolor{ourrowbg}\textbf{0.620} \\
 & PV      & 0.750 & \textbf{1.000} & \cellcolor{ourrowbg}0.917 \\
 & Mut     & 0.192 & 0.192 & \cellcolor{ourrowbg}\textbf{0.346} \\
\midrule
\multirow{3}{*}{\tblclaudesonnet{}}
 & Overall & 0.720 & \textbf{0.780} & \cellcolor{ourrowbg}\textbf{0.780} \\
 & PV      & 0.958 & \textbf{1.000} & \cellcolor{ourrowbg}\textbf{1.000} \\
 & Mut     & 0.500 & \textbf{0.577} & \cellcolor{ourrowbg}\textbf{0.577} \\
\midrule
\multirow{3}{*}{\tblgeminipro{}}
 & Overall & 0.480 & 0.600 & \cellcolor{ourrowbg}\textbf{0.680} \\
 & PV      & 0.750 & \textbf{1.000} & \cellcolor{ourrowbg}0.917 \\
 & Mut     & 0.231 & 0.231 & \cellcolor{ourrowbg}\textbf{0.462} \\
\bottomrule
\end{tabular}%
}
\caption{Agent-family generalization on Airline (50 tasks, $n{=}4$; verifier
model paired to the agent). All metrics are Pass$^4$. The
\tblgptfive{}-authored workflow graph is reused without re-authoring.}
\label{tab:modelgen}
\end{table}

\begin{figure}[t]
\centering
\pgfplotsset{
  mg axis/.style={
    passk base,
    width=0.56\linewidth, height=4.2cm,
    xmin=0.85, xmax=4.15, ymin=0.44, ymax=0.90,
    xtick={1,2,3,4}, xticklabels={$\textsc{P}^{1}$,$\textsc{P}^{2}$,$\textsc{P}^{3}$,$\textsc{P}^{4}$},
    ytick={0.5,0.6,0.7,0.8,0.9},
  },
}
\resizebox{\columnwidth}{!}{%
\begin{tikzpicture}
\begin{axis}[mg axis, title={Claude Sonnet 4.6}, name=mg1,
    legend to name=mglegend, legend columns=3, passk legend]
\addplot[p direct] coordinates {(1,0.815)(2,0.767)(3,0.740)(4,0.720)};
\addplot[p pg] coordinates {(1,0.845)(2,0.800)(3,0.785)(4,0.780)};
\addplot[p ours] coordinates {(1,0.855)(2,0.817)(3,0.795)(4,0.780)};
\addlegendimage{p direct}\addlegendentry{\direct{}}
\addlegendimage{p pg}\addlegendentry{\pgd{}}
\addlegendimage{p ours}\addlegendentry{\pg{}}
\end{axis}
\begin{axis}[mg axis, title={Gemini 2.5 Pro}, at={(mg1.south east)}, xshift=7mm, anchor=south west, yticklabels={,,}, name=mg2]
\addplot[p direct] coordinates {(1,0.700)(2,0.577)(3,0.515)(4,0.480)};
\addplot[p pg] coordinates {(1,0.695)(2,0.640)(3,0.615)(4,0.600)};
\addplot[p ours] coordinates {(1,0.780)(2,0.727)(3,0.700)(4,0.680)};
\end{axis}
\end{tikzpicture}%
}
\vspace{0.2em}\centerline{\ref{mglegend}}
\caption{Pass$^k$ for \tblclaudesonnet{} and \tblgeminipro{} agents on
airline (verifier $=$ agent), with $n{=}4$ for every system.}
\label{fig:modelgen}
\end{figure}

\subsection{Main results}
\label{sec:results}

\pg{} leads overall Pass$^4$ in all three domains and retains its lead as $k$
increases (Table~\ref{tab:main}; Figure~\ref{fig:main}); pooled paired tests
favor it over both baselines (Appendix~\ref{app:sig}). Against \direct{},
Pass$^4$ rises from $0.460$ to $0.620$ on Airline, $0.596$ to $0.614$ on Retail,
and $0.193$ to $0.614$ on Telecom. The largest gain aligns with Telecom's long
diagnose--instruct--verify chains. Gains span PV and Mut without trading
completion for stricter blocking.

On \retaildomain{}, \pg{} improves Mut while tying both baselines on PV.
The updated \pgd{} reaches $0.605$ overall Pass$^4$; \pg{}'s overall gains over
both baselines are not statistically significant (Appendix~\ref{app:sig}).

\subsection{Diagnostic workflow variants}
\label{sec:self}
\label{sec:ablation}

Table~\ref{tab:ablation} compares actor-only workflow access and compiled graph
structure.

\paragraph{Actor-only workflow access.}
\pgs{} gives the actor the frozen graph but removes the external verifier,
persisted state, remediation, and mutation intercept. Its Mut Pass$^4$ does not
exceed \direct{} in any domain: workflow access alone does not ensure reliable
execution.

\paragraph{Compiled structure under external tracking.}
\pgr{} retains the verifier schedule and remediation but replaces the graph
with raw policy. Against this matched guide, \pg{} improves
overall Pass$^4$ by $0.100$, $0.150$, and $0.325$ on Airline, Retail, and
Telecom. The larger Telecom gap is consistent with explicit graph position
helping the verifier resume long, ordered diagnostic chains.

\subsection{Matched workflow-controller comparison}
\label{sec:flowagent}

For the matched FlowAgent comparison (Table~\ref{tab:flowagent}), we
deterministically compile the same frozen graph into PDL without LLM authoring.
FlowAgent keeps the raw policy and PDL in the actor with its released
API-dependency and duplicate-call controllers; \pg{} uses an external verifier
with persisted graph state. \direct{} and \pgd{} provide actor-only and
action-local references.

\subsection{Generalization across agent families}
\label{sec:modelgen}

The frozen \tblgptfive{}-authored Airline graph also guides \tblclaudesonnet{}
and \tblgeminipro{} without re-authoring (Table~\ref{tab:modelgen};
Figure~\ref{fig:modelgen}); gains over unguided execution hold across all three
families. For \tblgeminipro{}, \pg{} improves Mut
Pass$^4$ from $0.231$ under either baseline to $0.462$, while its lower PV than
\pgd{} ($0.917$ versus $1.000$) indicates a completion benefit rather than
stricter final-action checking. Transfer across workflow-author models remains
untested.

\section{Adversarial and procedural validation}
\subsection{Adversarial robustness}
\label{sec:robust}

\begin{figure}[t]
\centering
\begin{tikzpicture}
\begin{axis}[
    passk base,
    width=0.86\columnwidth, height=4.6cm,
    xmin=0.85, xmax=4.15, ymin=0.05, ymax=0.28,
    xtick={1,2,3,4}, xticklabels={$k{=}1$,$k{=}2$,$k{=}3$,$k{=}4$},
    ytick={0.05,0.10,0.15,0.20,0.25},
    scaled y ticks=false, yticklabel style={/pgf/number format/fixed, /pgf/number format/precision=2},
    ylabel={\sffamily CRAFT ASR@$k$ ($\downarrow$ safer)},
    ylabel style={font=\scriptsize, yshift=-1.5mm},
    legend to name=asrlegend, legend columns=3, passk legend,
]
\addplot[p direct] coordinates {(1,0.200)(2,0.225)(3,0.237)(4,0.250)};
\addlegendentry{\direct{}}
\addplot[p pg] coordinates {(1,0.125)(2,0.175)(3,0.212)(4,0.250)};
\addlegendentry{\pgd{}}
\addplot[p ours] coordinates {(1,0.087)(2,0.142)(3,0.175)(4,0.200)};
\addlegendentry{\pg{}}
\end{axis}
\end{tikzpicture}
\vspace{0.2em}\centerline{\ref{asrlegend}}
\caption{CRAFT red-team attack-success rate on airline (20 attack tasks,
\gptfive{}, $n{=}4$); \emph{lower is safer}.}
\label{fig:asr}
\end{figure}

Under CRAFT~\citep{nakash2025craft}, persuasive users inject false eligibility
premises to induce forbidden mutations. \pg{} has the lowest ASR@$k$ at every
$k$ (Figure~\ref{fig:asr}); its per-trial ASR is $0.087$, versus $0.125$ for
\pgd{} and $0.200$ for \direct{}, preventing $91.3\%$ of tested attacks while
improving benign completion. These results show that \pg{} is robust to CRAFT
red-team attacks. Its lower ASR is consistent with requiring tool-derived
evidence, so unsupported user claims cannot satisfy workflow prerequisites.

\subsection{Procedural trace compliance}
\label{sec:trace-compliance}

\begin{table}[t]
\centering
\scriptsize
\setlength{\tabcolsep}{2.8pt}
\begin{tabular}{lrrr}
\toprule
System & Step-TCR & Trace-TCR & Process-valid rate \\
\midrule
\direct{} & 86.4 & 35.4 & 17.5 \\
\pgd{}    & 85.7 & 23.9 & 13.1 \\
\rowcolor{ourrowbg}
\textbf{\pg{}} & \textbf{94.5} & \textbf{63.4} & \textbf{56.2} \\
\bottomrule
\end{tabular}
\caption{Independently validated Telecom ordered trace compliance (\%; $n{=}4$).
Process-valid success uses all 160 task--trial pairs; Step- and Trace-TCR
condition on outcome-passing traces.}
\label{tab:telecom-trace-compliance}
\end{table}

Because final-state Pass ignores intermediate order, two authors independently
developed a task-conditioned Telecom rubric from the raw policy and support
manual in blinded reviews. Each separately specified the procedural conditions
and required event order; they then discussed their annotations and
cross-validated the final ordering. It checks identification, diagnosis before intervention, consent,
correction order, and final verification. On outcome-passing traces,
\emph{Step-TCR} is the fraction of applicable checks satisfied and
\emph{Trace-TCR} the fraction satisfying all checks. The \emph{process-valid
rate} uses all 160 task--run pairs and requires both the Tau2 outcome and every
applicable rubric check. It is our end-to-end measure, not a Pass$^k$ statistic;
conditional TCR compares different successful cohorts. This workflow-level extension of prerequisite
analysis is distinct from Call-NMR~\citep{rabinovich2026nearmiss,kang2026policyguard},
whose Telecom adaptation appears in Appendix~\ref{app:nearmiss}.

\pg{} attains the highest process-valid rate (56.2\%, versus 17.5\% for
\direct{} and 13.1\% for \pgd{}; Table~\ref{tab:telecom-trace-compliance}) and
leads both conditional diagnostics among these three main systems.
Appendix~\ref{app:trace-audit} reports the independent two-author validation
and counts for RAW/SELF.
A separate audit attributes 10/18 outcome failures to premature verifier transfer
and 8/18 to omitted actor transfer with a firing-boundary gap
(Appendix~\ref{app:residual-failures}).

\section{Conclusion}
\label{sec:conclusion}

\pg{}'s external verifier tracks a compiled policy workflow, persists progress,
and returns targeted remediation. Across three \tautwo{} domains, it achieves
the best overall Pass$^4$ and transfers across agent families; ablations and
independent ordered-trace validation support workflow-level guidance.


\section*{Limitations}

\paragraph{Evaluation and validation.}
Our evaluation covers all three English \tautwo{} customer-service domains
\citep{barres2025tau2}, with four trials per multi-trial cell and three actor
families on Airline. The procedural audit complements final-state outcomes:
two authors independently constructed Telecom rules and event orders in
blinded reviews, then discussed and cross-validated them
(Appendix~\ref{app:trace-audit}). Process-valid success uses all 160 task--trial
pairs per condition. The supplementary Call-NMR audit covers Airline and
Retail; the saturated Telecom oracle is reported as a non-identification result
(Appendix~\ref{app:nearmiss}). CRAFT tests the released 20-task Airline split
\citep{nakash2025craft}. Broader policy regimes, languages, live users, and
adaptive threats require setting-specific validation.

\paragraph{Workflow generation and faithfulness.}
One frozen \gptfive{}-authored workflow per domain gives every system and actor
family the same policy representation, isolating runtime and executor
differences from re-authoring. We manually verified each graph against its
source policy and tool specifications
(Appendix~\ref{app:workflow-validation}). This design supports a controlled
comparison and tests workflow reuse across actors; generalization of workflow
authoring across models or seeds is a separate question.

\paragraph{Trigger and cost.}
User-turn verification proactively guides procedural steps beyond mutating
calls, while the corrective intercept preserves the actor's execution loop
for remediation. The advisory write gate disarms after interception and
rearms at the next user turn. Appendix~\ref{app:trace-validation} characterizes
this behavior under forced retries, corrective-verifier exceptions, and
repeated transfers. The cadence concentrates verification at specified
boundaries; broader event coverage requires more verifier calls. Guide fees
are approximately $\$0.40$ per conversation; Appendix~\ref{app:cost} reports
call counts, fees, and observed latency. Smaller verifiers and more selective
invocation are possible deployment choices.

\paragraph{Probabilistic enforcement and fallback.}
\pg{} makes workflow progress and authorization explicit in code and checks
graph faithfulness through source validation. As with other LLM-based
safeguards \citep{chen2025shieldagent,xiang2025guardagent,kang2026policyguard},
node judgments remain probabilistic. Theorem~\ref{thm:complete-mediation}
states the verifier and coverage conditions for complete procedural validity;
the evaluated advisory configuration is assessed empirically. Fail-open
verifier handling preserves the execution loop on errors; persistent retry
interception, fail-closed handling, and deterministic monitors for formally
expressible requirements are extensions for deployments requiring stricter
mediation.

\section*{Ethics Statement}

\pg{} is a probabilistic aid for policy adherence, not a guarantee, and should
not be the sole control for high-stakes actions. Because the verifier reads the
conversation, tool results, and persisted workflow state, privacy, access
control, retention, and auditing requirements must extend to verifier calls and
logs. All experiments use synthetic \tautwo{} tasks and simulated users
\citep{barres2025tau2}; no real customer data or external actions are involved.
Generated workflows may reproduce source-policy errors or introduce unsupported
restrictions, so deployment requires review by policy owners, monitoring, and a
safe fallback. We will release the prompts, workflow schemas, and analysis
artifacts needed for reproducibility.

\bibliography{custom}

\clearpage
\appendix
\raggedbottom
\section{Theoretical Analysis}
\label{app:theory}

Policy workflows constrain a broader set of agent actions than the class
mediated by an action-triggered verifier. We formalize when a firing schedule
can preserve procedural validity over this broader workflow.

\subsection{Intervention coverage}
\label{app:theory-setup}

We represent each interaction as a sequence of policy-relevant events.  A
workflow $G$ defines a nonempty language $\mathcal{L}(G)$ of compliant complete
sequences.  We call a partial sequence valid when it can still be completed to
such a sequence, and write
\begin{equation}
  P_G = \operatorname{Pref}(\mathcal{L}(G))
      = \{\tau\mid \exists\rho:\tau\rho\in\mathcal{L}(G)\}.
  \label{eq:prefix-closure}
\end{equation}
Let $\Sigma_{\mathrm{ag}}$ denote policy-relevant agent actions, including
user-facing messages, instructions, and tool calls, and let
$A\subseteq\Sigma_{\mathrm{ag}}$ be the designated action class mediated by an
action-triggered verifier. For mutating-call guards, $A$ is the set of
agent-issued mutating tool calls.

A \emph{reachable first deviation} is a pair $(\tau,e)$ where $\tau\in P_G$
can arise during an interaction, $e\in\Sigma_{\mathrm{ag}}$ can be the next
agent event, and $\tau e\notin P_G$.  Let $D_G$ be the set of such deviations.
A firing schedule $S$ \emph{covers} $(\tau,e)\in D_G$ if it invokes the
verifier after observing $\tau$ and before $e$ is committed; let
$C_S(G)\subseteq D_G$ denote its \emph{intervention coverage}. We compare:
\begin{itemize}[itemsep=1pt,topsep=2pt,leftmargin=*]
  \item the \emph{action-triggered schedule} $S_A$, which fires
  exactly when the proposed event belongs to $A$; and
  \item the ideal \emph{workflow-level schedule} $S_{\mathrm{wf}}$, which fires
  before every policy-relevant agent action is committed.
\end{itemize}
The theorem compares these schedules under the same ideal verifier: when
consulted, it permits an event if and only if the resulting trace remains in
$P_G$, and a rejected event cannot commit.  User responses and tool results may
update the trace, but do not themselves violate an agent obligation.  An
uncovered agent event may commit without a verifier verdict.

\begin{theorem}[Complete intervention coverage]
\label{thm:complete-mediation}
Under the assumptions above, a firing schedule $S$ preserves procedural
validity ($P_G$ membership) throughout every execution if and only if
\begin{equation}
  C_S(G)=D_G.
  \label{eq:complete-mediation}
\end{equation}
That is, the verifier must cover every reachable first deviation.
\end{theorem}

\begin{proof}
For sufficiency, start from the empty valid prefix.  If the next observation is
not an agent event, it preserves $P_G$ by assumption.  If the agent proposes an
event that stays in $P_G$, committing it preserves validity.  Otherwise the
proposal is in $D_G$ and belongs to $C_S(G)$ by
Equation~\ref{eq:complete-mediation};
the ideal verifier rejects it before it commits.  Induction over the interaction
therefore keeps every prefix in $P_G$.

For necessity, suppose $(\tau,e)\in D_G$ is not in $C_S(G)$.  The valid prefix
$\tau$ can arise during an interaction, and the ideal verifier permits all
covered events that lead to it.  At $(\tau,e)$ the schedule does not fire, so
$e$ can commit and produce $\tau e\notin P_G$.  Thus $S$ cannot guarantee
procedural validity throughout every execution.
\end{proof}

\begin{corollary}[Workflow-level versus action-triggered coverage]
\label{cor:workflow-vs-action}
The workflow-level schedule covers every member of $D_G$. The action-triggered
schedule covers exactly those members whose action $e$ belongs to $A$. Therefore
workflow-level firing preserves procedural validity for every workflow, whereas
action-triggered firing does so if and only if every reachable first deviation
belongs to $A$. If $D_G$ contains a point with $e\notin A$, workflow-level
firing provides a guarantee that action-triggered firing cannot provide, even
when both use the same ideal verifier.
\end{corollary}

\begin{proof}
The workflow-level schedule fires before every agent-controlled action, so it
covers all of $D_G$. The action-triggered schedule fires exactly for actions in
$A$. The claims then follow directly from Theorem~\ref{thm:complete-mediation}.
\end{proof}

\begin{corollary}[Evaluated boundary schedule]
\label{cor:evaluated-schedule}
Let $S_{\mathrm{eval}}$ fire after every user turn and after the runtime
intercepts an unauthorized mutating call. Let $B_G\subseteq D_G$ contain the
reachable first deviations $(\tau,e)$ for which one of these firings occurs
after the valid prefix $\tau$ and before the next agent event $e$ is committed.
Under the same ideal, binding-verifier assumptions as
Theorem~\ref{thm:complete-mediation}, $S_{\mathrm{eval}}$ preserves procedural
validity if and only if $B_G=D_G$. In particular, if a first deviation can occur
after an intervening agent event and before the next scheduled firing, the
evaluated schedule does not provide an unconditional guarantee.
\end{corollary}

\begin{proof}
By construction, $C_{S_{\mathrm{eval}}}(G)=B_G$. The claim follows directly
from Theorem~\ref{thm:complete-mediation}.
\end{proof}

Corollary~\ref{cor:evaluated-schedule} characterizes the coverage condition for
the evaluated cadence; it does not assert that the deployed guide satisfies the
theorem's binding-verifier assumption. Its non-mutating remediation is advisory,
and the corrective mutation gate is one-shot, so the experiments measure risk
reduction under this practical schedule rather than a formal guarantee.

The timing in Theorem~\ref{thm:complete-mediation} is essential.  Because
$P_G$ is prefix-closed, once $\tau e\notin P_G$, no later extension can make
$\tau e$ a valid prefix: if some $\tau e\rho$ belonged to $P_G$, then its
prefix $\tau e$ would also belong to $P_G$. A later action-triggered verifier may
block a subsequent guarded action after reading the history, but it cannot
prevent or undo the earlier procedural violation.

\section{Policy structure analysis}
\label{app:policy}

This appendix supports the paper's central claim: \emph{workflow enforcement
matters most when the policy is itself a workflow, as in telecom.}  We extend
\pgd{}'s argument-/process-level classification with a third,
\emph{workflow-level} class, apply it to every atomic requirement of all
three source policies, and examine how both process- and workflow-level
requirements relate to the observed gains. Process-level requirements motivate
context-aware guidance generally, whereas workflow-level requirements create a
particular need for ordered graph traversal and persistent progress tracking.
Line references refer to the policy documents released with \tautwo{}: the
136-line retail policy and, for telecom, the 158-line account policy and the
205-line technical-support manual.

\subsection{Operational definitions}
\label{app:policy:method}

Following \pgd{}~\citep{kang2026policyguard}, we classify the \emph{source
policy document}, not any generated artefact of the systems under test.  A
requirement is \textbf{argument-level (A)} if verifiable from the mutating
call's arguments plus deterministic computation---the class precompiled
guards express natively---and \textbf{process-level (P)} if verification
must read the user--agent dialogue (\textbf{D}) and/or a prior read-only
tool result (\textbf{T}).  We split \pgd{}'s process-level class by adding
\textbf{workflow-level (W)}: a requirement that additionally mandates an
action ordered after the outcome of another required action (a state effect,
a user response to a prior step, or a post-act verification), so discharging
it out of order is itself a violation.  The three classes are mutually
exclusive: W rows read dialogue and tool results like P rows, but P is
reserved for the \emph{flat} remainder, dischargeable at the action point
from evidence gathered in any order (status checks,
elicit--confirm--act chains).  A and P{+}W therefore remain comparable with
\pgd{}.  Requirements are extracted by hand and grouped by each document's
own section headers; descriptive sentences and API meta-rules are excluded.
Table~\ref{tab:policy} gives the partition; the catalogs follow.

\begin{table}[h]
\centering
\small
\setlength{\tabcolsep}{4pt}
\resizebox{\columnwidth}{!}{%
\begin{tabular}{lrrrrrr}
\toprule
Domain & A & P & W & Total & \%\,P{+}W & \%\,W \\
\midrule
\tblairlinedomain{} & 14 & 27 & 2 & 43 & 67.4\% & 4.7\% \\
\tblretaildomain{} & 0  & 27 & 1 & 28 & ${\sim}100\%$ & 3.6\% \\
\tbltelecomdomain{} (main) & 1 & 21 & 7 & 29 & 96.6\% & 24.1\% \\
\tbltelecomdomain{} (manual) & 0 & 1 & 20 & 21 & 100\% & 95.2\% \\
\rowcolor{ourrowbg}
\textbf{\tbltelecomdomain{} (both)} & \textbf{1} & \textbf{22} & \textbf{27} & \textbf{50} & \textbf{98.0\%} & \textbf{54.0\%} \\
\bottomrule
\end{tabular}%
}
\caption{Argument- (A), process- (P), and workflow-level (W) partition of
the source policies (W splits \pgd{}'s process-level class; P{+}W equals
it).}
\label{tab:policy}
\end{table}

\begin{table*}[!t]
\centering
\small
\setlength{\tabcolsep}{5pt}
\begin{tabular}{@{}l l p{10.6cm} l@{}}
\toprule
\textbf{ID} & \textbf{Line} & \textbf{Requirement (paraphrased)} & \textbf{Type} \\
\midrule
\midrule
\multicolumn{4}{@{}l@{}}{\textit{Global rules}} \\
G1 & 10  & Authenticate identity by locating the user id via email or name{+}zip---even when the user already provides the id & P (D+T) \\
G2 & 14  & One user per conversation; deny any request about another user                                    & P (D)   \\
G3 & 16  & List action details + obtain explicit ``yes'' before any DB-updating action                        & P (D)   \\
G4 & 18  & No fabricated information/knowledge/procedures; no subjective recommendations                      & P (D)   \\
G5 & 20  & At most one tool call per turn (not paired with a user-facing reply)                               & P (D)   \\
G6 & 22  & Deny user requests that are against the policy                                                     & P (D)   \\
G7 & 24  & Transfer iff unhandleable: call \texttt{transfer\_to\_human\_agents}, then the literal handoff message & W (D) \\
\midrule
\multicolumn{4}{@{}l@{}}{\textit{Generic action rules}} \\
N1 & 82  & Act only on orders with status \texttt{pending} or \texttt{delivered}                              & P (T)   \\
N2 & 84  & Exchange / modify-items tools callable only once per order                                         & P (T)   \\
N3 & 84  & Collect \emph{all} items to change into one list before making the call                            & P (D)   \\
\midrule
\multicolumn{4}{@{}l@{}}{\textit{Cancel pending order}} \\
C1 & 88  & Order status must be \texttt{pending}; check it before taking the action                           & P (T)   \\
C2 & 90  & User confirms order id + reason $\in$ \{`no longer needed', `ordered by mistake'\}; no other reason & P (D)  \\
\midrule
\multicolumn{4}{@{}l@{}}{\textit{Modify pending order}} \\
M1 & 96  & Order status must be \texttt{pending}; check it before taking the action                           & P (T)   \\
M2 & 98  & Only shipping address, payment method, or item options may be modified---nothing else              & P (D)   \\
M3 & 102 & New payment = a single method, different from the original                                          & P (T)   \\
M4 & 104 & If the new payment is a gift card, its balance must cover the total amount                          & P (T)   \\
M5 & 110 & Modify-items is one-shot (order becomes unmodifiable): remind + confirm all items first             & P (D)   \\
M6 & 112 & Each new item must be available                                                                     & P (T)   \\
M7 & 112 & New item = same product, different option (no product-type change)                                  & P (T)   \\
M8 & 114 & User provides a payment method for the price difference                                             & P (D)   \\
M9 & 114 & If that payment is a gift card, its balance must cover the price difference                         & P (T)   \\
\midrule
\multicolumn{4}{@{}l@{}}{\textit{Return delivered order}} \\
R1 & 118 & Order status must be \texttt{delivered}; check it before taking the action                          & P (T)   \\
R2 & 120 & User confirms order id + the list of items to be returned                                           & P (D)   \\
R3 & 122--124 & Refund method provided; must be the original payment method or an existing gift card            & P (T)   \\
\midrule
\multicolumn{4}{@{}l@{}}{\textit{Exchange delivered order}} \\
E1 & 130 & Order status must be \texttt{delivered}; check it before taking the action                          & P (T)   \\
E2 & 130 & Remind + confirm the user has provided \emph{all} items to exchange (one-shot)                      & P (D)   \\
E3 & 132 & Each new item = same product, different option, and available                                       & P (T)   \\
E4 & 134 & Payment for the price difference; if a gift card, balance must cover the difference                 & P (T)   \\
\bottomrule
\end{tabular}
\caption{Hand-classified atomic requirements of the \tautwo{}
\tblretaildomain{} policy
document (28 requirements, $0$\,A / $27$\,P / $1$\,W; subtypes D-only $13$, T-only
$14$, D+T $1$). \textit{Line} refers to \texttt{retail/policy.md} as released
with \tautwo{}; \textit{Type} A = argument-level, P = process-level (flat),
W = workflow-level (order-bound; Appendix~\ref{app:policy:method}), with \textbf{D}
= dialogue-dependent, \textbf{T} = requires a prior read-only tool call.}
\label{tab:cat-retail}
\end{table*}

\begin{table*}[!t]
\centering
\small
\setlength{\tabcolsep}{5pt}
\begin{tabular}{@{}l l p{10.6cm} l@{}}
\toprule
\textbf{ID} & \textbf{Line} & \textbf{Requirement (paraphrased)} & \textbf{Type} \\
\midrule
\midrule
\multicolumn{4}{@{}l@{}}{\textit{Global rules}} \\
G1 & 7   & No fabricated information/knowledge/procedures; no subjective recommendations                     & P (D)   \\
G2 & 9   & At most one tool call per turn (not paired with a user-facing reply)                              & P (D)   \\
G3 & 11  & Deny user requests that are against the policy                                                    & P (D)   \\
G4 & 13  & Transfer iff unhandleable: call \texttt{transfer\_to\_human\_agents}, then the literal handoff message & W (D) \\
G5 & 15  & Try your best to resolve the issue before transferring                                            & W (D)   \\
\midrule
\multicolumn{4}{@{}l@{}}{\textit{Customer lookup}} \\
L1 & 94--97 & Identify the customer via phone number, customer ID, or full name + date of birth              & P (D+T) \\
L2 & 99  & For name lookup, date of birth is required for verification                                       & P (D)   \\
\midrule
\multicolumn{4}{@{}l@{}}{\textit{Overdue bill payment (ordered procedure)}} \\
O1 & 105, 117 & \circled{1} Check the bill status \emph{is} \texttt{Overdue} before acting (the API does not check it) & P (T) \\
O2 & 106 & \circled{2} Check the bill amount due                                                             & P (T)   \\
O3 & 107--108 & \circled{3} Send the payment request ($\to$ \texttt{AWAITING PAYMENT}); gated on O1          & P (T)   \\
O4 & 109--110 & \circled{4} Inform the user to check their payment requests                                  & W (D)   \\
O5 & 111 & \circled{5} Only after the user accepts, call \texttt{make\_payment}                              & W (D+T) \\
O6 & 113 & \circled{6} \emph{Always verify} the bill became \texttt{PAID} before telling the user            & W (T)   \\
O7 & 116 & At most one bill in \texttt{AWAITING PAYMENT} at a time                                           & P (T)   \\
\midrule
\multicolumn{4}{@{}l@{}}{\textit{Line suspension}} \\
S1 & 125 & Lift a suspension only after all overdue bills are paid                                           & W (T)   \\
S2 & 126 & Do not lift if the contract end date is past---even if all bills are paid                         & P (T)   \\
S3 & 128 & After resuming, instruct the user to reboot the device                                            & W (D)   \\
\midrule
\multicolumn{4}{@{}l@{}}{\textit{Data refueling (ordered procedure)}} \\
F1 & 134 & Refuel amount $\le 2$\,GB                                                                          & \textbf{A} \\
F2 & 136 & \circled{1} Ask how much data the user wants to refuel                                            & P (D)   \\
F3 & 137 & \circled{2} Confirm the price                                                                     & P (D)   \\
F4 & 138 & \circled{3} Apply the refuel to the line associated with the user's phone number                  & P (D+T) \\
\midrule
\multicolumn{4}{@{}l@{}}{\textit{Change plan (ordered procedure)}} \\
P1 & 144 & \circled{1} Establish which line the plan change is for                                           & P (D)   \\
P2 & 145 & \circled{2} Gather the available plans                                                            & P (T)   \\
P3 & 146 & \circled{3} Ask the user to select one                                                            & P (D)   \\
P4 & 147 & \circled{4} Calculate the price of the new plan                                                   & P (T)   \\
P5 & 148 & \circled{5} Confirm the price                                                                     & P (D)   \\
P6 & 149 & \circled{6} Apply the plan to the line associated with the user's phone number                    & P (D+T) \\
\midrule
\multicolumn{4}{@{}l@{}}{\textit{Data roaming}} \\
RM1 & 155 & If the user is travelling abroad, check whether the line is roaming-enabled                      & P (T)   \\
RM2 & 155 & If not enabled, enable it at no cost                                                              & P (T)   \\
\bottomrule
\end{tabular}
\caption{Hand-classified atomic requirements of the \tautwo{}
\tbltelecomdomain{}
\texttt{main\_policy.md} (29 requirements, $1$\,A / $21$\,P / $7$\,W). \textit{Line}
refers to the document as released with \tautwo{}; \circled{$n$} marks a step
in an ordered procedure (``To do so you need to follow these steps''). Types as
in Table~\ref{tab:cat-retail}.}
\label{tab:cat-telecom-main}
\end{table*}

\begin{table*}[!t]
\centering
\small
\setlength{\tabcolsep}{5pt}
\begin{tabular}{@{}l l p{10.2cm} l@{}}
\toprule
\textbf{ID} & \textbf{Line} & \textbf{Requirement (diagnose $\to$ conditional fix $\to$ verify)} & \textbf{Type} \\
\midrule
\midrule
\multicolumn{4}{@{}l@{}}{\textit{Cellular service (ll.\,55--99)}} \\
TSS1 & 69--72 & Diagnose service via \texttt{check\_status\_bar}                                             & P (T)   \\
TSS2 & 74--78 & If Airplane Mode ON $\to$ guide \texttt{toggle\_airplane\_mode} OFF                          & W (D+T) \\
TSS3 & 79--87 & Check SIM: \texttt{Missing} $\to$ reseat; \texttt{Locked} $\to$ escalate; \texttt{Active} $\to$ ok (three-way branch) & W (D+T) \\
TSS4 & 88--92 & If APN incorrect $\to$ guide \texttt{reset\_apn\_settings}, \emph{then} \texttt{reboot\_device} & W (D+T) \\
TSS5 & 93--99 & If line suspended $\to$ handle per main policy, then verify service restored                 & W (T)   \\
\midrule
\multicolumn{4}{@{}l@{}}{\textit{Mobile data (ll.\,100--163)}} \\
TSD0 & 106--108 & \emph{Prerequisite}: the user must first have cellular service                            & W (T)   \\
TSD1 & 122--127 & Diagnose via \texttt{run\_speed\_test}                                                    & W (T)   \\
TSD2 & 129--131 & Airplane Mode (as in the Service chapter)                                                 & W (D+T) \\
TSD3 & 132--135 & If mobile data disabled $\to$ guide \texttt{toggle\_data} ON                              & W (D+T) \\
TSD4 & 136--141 & If roaming abroad \& data off $\to$ guide \texttt{toggle\_roaming} + verify the line is roaming-enabled & W (D+T) \\
TSD5 & 142--145 & If Data Saver ON $\to$ guide \texttt{toggle\_data\_saver\_mode} OFF                       & W (D+T) \\
TSD6 & 146--150 & If VPN ON \& performance poor $\to$ guide \texttt{disconnect\_vpn}                        & W (D+T) \\
TSD7 & 151--158 & If usage exceeds the plan limit $\to$ offer change-plan or refuel                         & W (T)   \\
TSD8 & 159--163 & If network mode 2G/3G $\to$ guide \texttt{set\_network\_mode\_preference}                 & W (D+T) \\
\midrule
\multicolumn{4}{@{}l@{}}{\textit{MMS (ll.\,164--205)}} \\
TSM0 & 170--173 & \emph{Prerequisite}: the user must have cellular service \emph{and} mobile data           & W (T)   \\
TSM1 & 181--183 & Diagnose via \texttt{can\_send\_mms}                                                      & W (T)   \\
TSM2 & 185--188 & Ensure basic service + data connectivity first                                            & W (T)   \\
TSM3 & 189--193 & If on 2G $\to$ guide \texttt{set\_network\_mode\_preference} to 3G+                       & W (D+T) \\
TSM4 & 194--199 & If MMSC URL unset $\to$ guide \texttt{reset\_apn\_settings}, \emph{then} \texttt{reboot\_device} & W (D+T) \\
TSM5 & 200--203 & If Wi-Fi Calling ON $\to$ guide \texttt{toggle\_wifi\_calling} OFF                        & W (D+T) \\
TSM6 & 204--205 & If the messaging app lacks storage/SMS permissions $\to$ guide \texttt{grant\_app\_permission} & W (D+T) \\
\bottomrule
\end{tabular}
\caption{Hand-classified atomic requirements of the \tautwo{}
\tbltelecomdomain{}
\texttt{tech\_support\_manual.md} (21 requirements, $1$\,P / $20$\,W). \textit{Line}
refers to the document as released with \tautwo{}. Every rule is a
diagnostic-gated (T) user-guidance (D) step; the three chapters form the
prerequisite chain Service $\subset$ Data $\subset$ MMS; every row except
TSS1 (the entry diagnostic) is workflow-level.}
\label{tab:cat-telecom-manual}
\end{table*}

\subsection{Airline catalog (summary)}
\label{app:policy:airline}

The airline classification is inherited from \pgd{}~\citep{kang2026policyguard}
(full catalog there); we add the workflow-level split.  The 167-line policy
yields 43 requirements, $14$\,A / $27$\,P / $2$\,W.  Argument-level mass sits
in booking and modification schema rules (cabin uniformity, passenger limits,
payment-method counts); process-level mass splits between dialogue
obligations (explicit confirmation, the insurance offer) and tool-read
eligibility gates (baggage allowances, flown-segment checks, the disjunctive
cancellation and compensation conditions).  The two W rows are the transfer
pair and the delayed-flight certificate mandated after a change or
cancellation---otherwise airline is flat gates discharged at the mutation.

\subsection{Retail catalog}
\label{app:policy:retail}

The 28 retail requirements (Table~\ref{tab:cat-retail}) partition as $0$\,A /
$27$\,P / $1$\,W (13 D, 14 T, 1 D+T).  Every candidate A-row ranges over
environment state rather than argument values: the status gates read the
order-status enum only \texttt{get\_order\_details} surfaces, the balance and
variant constraints compare against profile and catalog reads, and identity
must be re-derived even when the user supplies an id.  The contrast with
airline is mechanical: airline's \texttt{book\_reservation} passes the whole
reservation as call arguments, so schema constraints are argument-checkable,
whereas retail and telecom mutations are thin id-referencing calls whose
constrained values are surfaced only by read calls.

Structurally, however, retail is the flattest domain: each mutation is
guarded by an order-free conjunction---status $\wedge$ content $\wedge$
payment $\wedge$ confirmation---with prerequisite depth 1, no branching, and
no verify-after-act; its only W row is the transfer pair.  This is exactly
the regime a conversation-aware pass/block verifier already covers, and why
retail is where \pg{}'s graph adds the least.

\subsection{Telecom catalog}
\label{app:policy:telecom}

The 29 account-policy requirements (Table~\ref{tab:cat-telecom-main})
partition as $1$\,A / $21$\,P / $7$\,W; the lone A row is the refuel-amount
bound.  The W rows concentrate in the overdue-payment sequence: send the
payment request only for a confirmed-\texttt{Overdue} bill, call
\texttt{make\_payment} only after the user accepts the request it created,
and \emph{always} verify the bill became \texttt{PAID}---prerequisite
chaining that ends in a verify-after-act obligation a pre-execution verifier
structurally cannot enforce.  Suspension adds cross-procedure dependence
(lift only after the overdue bills are paid, unless the contract has ended)
and a post-act duty only the user can perform (reboot the device).

\label{app:policy:manual}
The 205-line technical-support manual is a troubleshooting manual rather
than a rulebook: 21 requirements (Table~\ref{tab:cat-telecom-manual}),
$1$\,P / $20$\,W.  Its defining property is \emph{dual control}: fixes
execute on the user's device, so the agent must instruct the user, await
their report, and re-verify.  Every rule instantiates
diagnose~$\to$~conditional fix~$\to$~re-verify, and the chapters impose the
prerequisite chain Service~$\subset$~Data~$\subset$~MMS---a literal decision
tree with mandated traversal order, on which a task may contain \emph{no
agent-side mutating call to intercept} at all.

\subsection{Policy structure and observed gains}
\label{app:policy:axis2}

Every domain is majority process-level, consistent with the benefit of
context-aware guidance over naive acting. Yet retail---the most process-level
domain---gains \emph{least}, indicating that the P{+}W fraction alone does not
explain the cross-domain variation. Workflow-level requirements are much more
concentrated in telecom: $2/43$ on airline and $1/28$ on retail versus $27/50$
on telecom, including $20/21$ in the manual alone
(Table~\ref{tab:policy}). The rule is applied uniformly: telecom's refuel and
change-plan recipes are
elicit--confirm--apply chains and remain P, the same shallow shape as
retail; its W mass lives in the overdue-payment state machine, the
suspension procedure, and above all the diagnostic manual.  \pg{} tracks
position in a workflow graph across turns, so it has the most to exploit
exactly where W requirements concentrate. Consistent with this account, the
compiled-graph gain over the matched raw-policy guide is larger on telecom
($+0.325$) than on airline ($+0.100$) or retail ($+0.150$;
\S\ref{sec:ablation}). Thus process-level
requirements help explain the general value of context-aware guidance, while
the concentration of workflow-level requirements helps explain why \pg{}'s
explicit graph and state tracking are especially useful in telecom.

\section{Reliability and significance}
\label{app:sig}

\paragraph{Pass$^k$ breakdown.}
Table~\ref{tab:passk-breakdown} gives the exact values underlying
Figure~\ref{fig:main}. Table~\ref{tab:trial-variance} reports Pass$^1$ separately
for each trial. These tables use the same base splits as Table~\ref{tab:main}:
Airline base-50 and Retail/Telecom base-114. Retail \pgd{} uses the
measured fixed-scope checklist run. Task pairing compares saved runs rather
than identical simulator draws.

\begin{table}[t]
\centering
\scriptsize
\setlength{\tabcolsep}{2.5pt}
\resizebox{\columnwidth}{!}{%
\begin{tabular}{llccccc}
\toprule
Domain & System & $\mathrm{P}^{1}$ & $\mathrm{P}^{2}$ & $\mathrm{P}^{3}$ & $\mathrm{P}^{4}$ & $\mathrm{P}^{4}/\mathrm{P}^{1}$ \\
\midrule
\multirow{4}{*}{\tblairlinedomain}
 & \direct{} & 0.640 & 0.530 & 0.485 & 0.460 & 0.72 \\
 & \tg{}     & 0.575 & 0.553 & 0.535 & 0.520 & 0.90 \\
 & \pgd{}    & 0.710 & 0.630 & 0.595 & 0.580 & 0.82 \\
\rowcolor{ourrowbg}
\cellcolor{white} & \textbf{\pg{}} & \textbf{0.775} & \textbf{0.707} & \textbf{0.660} & \textbf{0.620} & 0.80 \\
\midrule
\multirow{3}{*}{\tblretaildomain}
 & \direct{} & 0.800 & 0.700 & 0.638 & 0.596 & 0.75 \\
 & \pgd{}    & 0.800 & 0.703 & 0.643 & 0.605 & 0.76 \\
\rowcolor{ourrowbg}
\cellcolor{white} & \textbf{\pg{}} & \textbf{0.809} & \textbf{0.715} & \textbf{0.654} & \textbf{0.614} & 0.76 \\
\midrule
\multirow{3}{*}{\tbltelecomdomain}
 & \direct{} & 0.384 & 0.273 & 0.226 & 0.193 & 0.50 \\
 & \pgd{}    & 0.406 & 0.292 & 0.237 & 0.202 & 0.50 \\
\rowcolor{ourrowbg}
\cellcolor{white} & \textbf{\pg{}} & \textbf{0.866} & \textbf{0.763} & \textbf{0.682} & \textbf{0.614} & 0.71 \\
\bottomrule
\end{tabular}%
}
\caption{Pass$^k$ breakdown for the base-split results in
Table~\ref{tab:main} and Figure~\ref{fig:main} (\gptfive{}, $n{=}4$).
$\mathrm{P}^{4}/\mathrm{P}^{1}$ is the consistency ratio.}
\label{tab:passk-breakdown}
\end{table}

\begin{table}[t]
\centering
\scriptsize
\setlength{\tabcolsep}{2.5pt}
\begin{tabular}{llccccc}
\toprule
Domain & System & T1 & T2 & T3 & T4 & pstd \\
\midrule
\multirow{4}{*}{\tblairlinedomain}
 & \direct{} & 0.620 & 0.620 & 0.640 & 0.680 & 0.024 \\
 & \tg{}     & 0.560 & 0.580 & 0.580 & 0.580 & 0.009 \\
 & \pgd{}    & 0.700 & 0.740 & 0.700 & 0.700 & 0.017 \\
\rowcolor{ourrowbg}
\cellcolor{white} & \textbf{\pg{}} & 0.800 & 0.720 & 0.820 & 0.760 & 0.038 \\
\midrule
\multirow{3}{*}{\tblretaildomain}
 & \direct{} & 0.807 & 0.746 & 0.842 & 0.807 & 0.035 \\
 & \pgd{}    & 0.807 & 0.807 & 0.807 & 0.781 & 0.011 \\
\rowcolor{ourrowbg}
\cellcolor{white} & \textbf{\pg{}} & 0.816 & 0.798 & 0.789 & 0.833 & 0.017 \\
\midrule
\multirow{3}{*}{\tbltelecomdomain}
 & \direct{} & 0.342 & 0.377 & 0.404 & 0.412 & 0.027 \\
 & \pgd{}    & 0.465 & 0.386 & 0.360 & 0.412 & 0.039 \\
\rowcolor{ourrowbg}
\cellcolor{white} & \textbf{\pg{}} & 0.860 & 0.860 & 0.842 & 0.904 & 0.023 \\
\bottomrule
\end{tabular}
\caption{Pass$^1$ in each of the four trials on the base splits. pstd is the
population standard deviation across trial-level values.}
\label{tab:trial-variance}
\end{table}

\paragraph{Paired significance.}
For each task, Pass$^4$ is one iff all four trials succeed. We compare systems
on common tasks and obtain $95\%$ confidence intervals by paired bootstrap
(10{,}000 task-level resamples). Across domain strata, we use
\[
Z=\frac{\sum_d a_d-\sum_d b_d}
        {\sqrt{\sum_d a_d+\sum_d b_d}},
\]
where $a_d$ counts \pg{}-only Pass$^4$ successes and $b_d$ the reverse.
Table~\ref{tab:sig-pooled} shows that \pg{} improves pooled Pass$^4$ over
\direct{} ($p<10^{-8}$) and \pgd{} ($p<10^{-7}$), pooling all three domains.
Domain-level effect sizes and intervals appear in Table~\ref{tab:sig-ci}.
Because $\Delta\mathrm{P}^{4}$ is a signed difference rather than a probability,
negative limits are valid; an interval crossing zero indicates that the
domain-level difference is not statistically distinguishable from zero.

\begin{table}[t]
\centering
\scriptsize
\setlength{\tabcolsep}{2.5pt}
\begin{tabular}{lrrrrrr}
\toprule
Opponent & $D$ & $\sum a$ & $\sum b$ & $n_{\rm disc}$ & $Z$ & $p$ \\
\midrule
\direct{} & 3 & 77 & 19 & 96  & $+5.92$ & $<10^{-8}$ \\
\pgd{}    & 3 & 80 & 24 & 104 & $+5.49$ & $<10^{-7}$ \\
\bottomrule
\end{tabular}
\caption{Pooled stratified McNemar tests on per-task Pass$^4$. $D$ is the
number of domain strata; $a$ counts \pg{}-only passes and $b$ the reverse,
summed across strata.}
\label{tab:sig-pooled}
\end{table}

\begin{table}[t]
\centering
\scriptsize
\setlength{\tabcolsep}{3pt}
\begin{tabular}{llrc}
\toprule
Domain & Opponent & $n$ & \textbf{$\Delta\mathrm{P}^{4}$ [95\% CI]} \\
\midrule
\multirow{3}{*}{\tblairlinedomain}
 & \direct{} & 50  & $+0.160$ [$+0.020$, $+0.300$] \\
 & \tg{}     & 50  & $+0.100$ [$-0.040$, $+0.260$] \\
 & \pgd{}    & 50  & $+0.040$ [$-0.080$, $+0.160$] \\
\midrule
\multirow{2}{*}{\tblretaildomain}
 & \direct{} & 114 & $+0.018$ [$-0.070$, $+0.105$] \\
 & \pgd{}    & 114 & $+0.009$ [$-0.096$, $+0.114$] \\
\midrule
\multirow{2}{*}{\tbltelecomdomain}
 & \direct{} & 114 & $+0.421$ [$+0.316$, $+0.526$] \\
 & \pgd{}    & 114 & $+0.412$ [$+0.298$, $+0.526$] \\
\bottomrule
\end{tabular}
\caption{Per-domain paired-bootstrap differences in Pass$^4$ on the base
splits (10{,}000 task-level resamples). Positive values favor \pg{}.}
\label{tab:sig-ci}
\end{table}

\section{Cost analysis}
\label{app:cost}

\pg{} adds verifier inference to the underlying agent. We therefore report
guide-side model usage separately from the actor and user simulator, and examine
how prompt caching and the firing policy limit this overhead.

\subsection{Guide-side usage}
\label{app:cost:measured}

Table~\ref{tab:cost} summarizes the \gptfive{} configuration. The guide costs
\$0.34--\$0.56 per task and fires 7.4--11.5 times per task; Telecom is higher
because its diagnostic workflows require longer interactions. Although
85.8--88.1\% of prompt tokens are cached, each call generates about 2.2--2.5k
output tokens for workflow traversal, evidence checks, and remediation. Output
generation consequently accounts for an estimated 67.5--71.0\% of guide spend.

\begin{table}[h]
\centering
\scriptsize
\setlength{\tabcolsep}{2pt}
\begin{tabular}{lrrrrrr}
\toprule
Domain & \makecell{Calls/\\task} & \makecell{Prompt\\tok./call} &
\makecell{Cached\\input} & \makecell{Output\\tok./call} &
\makecell{Guide\\total \$} & \makecell{Guide\\\$/task} \\
\midrule
\tblairlinedomain{} & 7.56  & 32,360 & 88.1\% & 2,478 & 20.10 & 0.40 \\
\tblretaildomain{}  & 7.42  & 22,803 & 85.8\% & 2,179 & 13.67 & 0.34 \\
\tbltelecomdomain{} & 11.47 & 28,518 & 86.5\% & 2,186 & 22.29 & 0.56 \\
\bottomrule
\end{tabular}
\caption{Guide-side model usage for the \gptfive{} configuration (50 Airline
and 40 Retail/Telecom tasks). Costs exclude the actor and user simulator.}
\label{tab:cost}
\end{table}

Thus, caching substantially reduces repeated input processing, but does not
eliminate the marginal cost of the verifier: its structured audit is much longer
than a binary policy verdict. Reducing guide output length is therefore the main
remaining cost-optimization opportunity.

\subsection{Wall-clock time}
\label{app:cost:latency}

Table~\ref{tab:latency} reports mean end-to-end task time. \pg{} requires
$5.45$--$5.78\times$ the observed wall-clock time of \direct{}, reflecting the
additional verifier generations at successive turns.

\begin{table}[h]
\centering
\small
\setlength{\tabcolsep}{4pt}
\begin{tabular}{lrrr}
\toprule
Domain & \makecell{\direct{}\\(s/task)} & \makecell{\pg{}\\(s/task)} & Ratio \\
\midrule
\tblairlinedomain{} & 36.4 & 210.1 & $5.78\times$ \\
\tblretaildomain{}  & 34.6 & 193.6 & $5.60\times$ \\
\tbltelecomdomain{} & 45.5 & 247.6 & $5.45\times$ \\
\bottomrule
\end{tabular}
\caption{Mean end-to-end wall-clock time per task.}
\label{tab:latency}
\end{table}

The measurement covers the complete simulated conversation, including actor,
verifier, user-simulator, and tool execution, rather than isolated verifier
latency.

\subsection{Cost-aware execution}
\label{app:cost:cache}

The verifier prompt places the policy, workflow graph, tool specifications,
judging rules, and output contract in a byte-stable prefix. Only the evolving
conversation and latest request state vary across calls, allowing the repeated
enforcement context to benefit from prefix caching.

Each firing uses one verifier generation for all open requests and may advance
across several satisfied workflow nodes. The verifier fires before responses to
user turns, while intervening tool results are incorporated at the next firing;
an intercepted unauthorized action triggers an additional check. Consequently,
the number of guide calls scales with relevant agent turns rather than with
individual workflow nodes or tool observations.

\section{Workflow verification}
\label{app:workflow-validation}

After generation, the authors manually verified each frozen workflow against the
source policy and tool specifications. We reviewed the represented request
types, the ordering of policy prerequisites, the authorization and subsequent
verification of mutating actions, and the policy or tool-contract basis of graph
constraints. This was a verification step: we did not manually edit the
generated workflows used in the experiments.

Programmatic validation is also integrated into workflow generation. Each
generated file is schema-validated, and the assembled graph is checked for
subflow composition, valid tool references and decision branches, reachability,
and mutating-action authorization coverage. The resulting findings are supplied
to the pipeline's automated review stage. We reran these checks on the exact
frozen workflows; Table~\ref{tab:workflow-validation} reports the results.

\begin{table}[h]
\centering
\small
\setlength{\tabcolsep}{4pt}
\begin{tabular}{lrrr}
\toprule
Domain & Nodes & Auth. nodes & Validator flags \\
\midrule
\tblairlinedomain{} & 158 & 11 & 0 \\
\tblretaildomain{}  & 104 &  7 & 0 \\
\tbltelecomdomain{} & 127 &  5 & 1 \\
\bottomrule
\end{tabular}
\caption{Programmatic validation rerun on the frozen workflow graphs.}
\label{tab:workflow-validation}
\end{table}

The Telecom flag concerns \texttt{disable\_roaming}, which is exposed by the
environment but has no authorizing workflow path. Manual review confirmed that
the source policy specifies enabling roaming but does not authorize the agent to
disable it. Its absence therefore does not omit a source-policy procedure; the
workflow correctly leaves this action outside its authorized policy scope.

\section{Call-level near-miss audit}
\label{app:nearmiss}

\begin{table}[h]
\centering
\small
\setlength{\tabcolsep}{4pt}
\begin{tabular}{lccc}
\toprule
Call-NMR (\%; $\downarrow$) & Airline & Retail & Telecom$^\dagger$ \\
\midrule
\direct{} & 25.4 & 47.6 & 0.0 \\
\pgd{}    & 32.5 & 47.3 & 0.0 \\
\rowcolor{ourrowbg}
\textbf{\pg{}} & \textbf{15.6} & \textbf{34.7} & 0.0 \\
\bottomrule
\end{tabular}
\caption{Call-NMR on passing Mut trajectories ($n{=}4$): percentage of
successfully executed agent mutations missing a frozen guard-derived read
prerequisite. Retail uses the same held-out 40 tasks; \pgd{} uses the
updated fixed-scope run. $^\dagger$Telecom is an adapted, agent-side diagnostic whose
read oracle saturates; its zeros do not establish equal procedural quality.}
\label{tab:nearmiss}
\end{table}

A \emph{near miss}~\citep{rabinovich2026nearmiss} is a mutation in an
outcome-passing task that lacks a policy prerequisite. Following the
runtime-view convention used by \pgd{},
Call-NMR~\citep{rabinovich2026nearmiss,kang2026policyguard} is the fraction of
successfully executed mutating calls in passing Mut trajectories that lack at
least one earlier read required by a frozen ToolGuard guard. Blocked attempts,
tool errors, and calls without a successful response are excluded. The same
domain oracle is applied to every system.

On Airline, \pg{} has the lowest observed rate (15.6\%, versus 25.4\% for
\direct{} and 32.5\% for \pgd{}). On Retail's held-out 40 tasks (3 PV / 37
Mut under the main taxonomy), \pg{} scores 34.7\% (59/170 calls) versus
47.3\% (80/169) for the updated \pgd{}, with 120 versus 117 outcome-passing
Mut trajectories. Thus Call-NMR audits prior-read coverage
conditional on success, not task coverage or complete procedural validity.

\paragraph{Telecom adaptation.}
The original guard-derived audit is not defined for Telecom. We adapt its
call-level convention using a frozen \gptfive{} guard tree generated from the
tagged concatenation of Telecom's raw agent policy and technical-support
manual. Argument- and response-aware matching requires reads to resolve the
same customer, line, or bill as the mutation. Because Telecom is dual-control,
the denominator covers only agent-side carrier mutations; user/device actions
such as toggling data, resetting APN settings, and rebooting are outside the
agent-call oracle.

The adapted read oracle yields 0.0\% for all three primary systems
(Table~\ref{tab:nearmiss}). This is a ceiling effect, not evidence that they are
procedurally equivalent: the oracle cannot express conversational evidence
such as travel status, selected refuel amount, and price confirmation, nor
ordering among user/device actions. This non-identification motivates our
workflow-level expansion of prerequisite analysis in
Section~\ref{sec:trace-compliance}.

\section{Prompts}
\label{app:prompts}

\definecolor{cardteal}{HTML}{1F8C9C}
\definecolor{cardslate}{HTML}{5B6470}
\lstdefinestyle{promptbody}{basicstyle=\ttfamily\scriptsize, frame=none,
  xleftmargin=0pt, xrightmargin=0pt, aboveskip=1pt, belowskip=1pt}
\newtcolorbox{guidecard}[2][]{enhanced, colback=cardteal!3!white,
  colframe=cardteal!45!black!35, boxrule=0.5pt, arc=1.6mm,
  left=5pt, right=5pt, top=7pt, bottom=4pt, breakable=false,
  title={#2}, fonttitle=\sffamily\bfseries\scriptsize, coltitle=white,
  colbacktitle=cardteal!85!black,
  attach boxed title to top left={yshift=-2.4mm, xshift=3mm},
  boxed title style={arc=0.8mm, boxrule=0pt, left=5pt, right=5pt, top=1.6pt, bottom=1.6pt},
  #1}
\newtcolorbox{gencard}[2][]{enhanced, colback=cardslate!4!white,
  colframe=cardslate!55, boxrule=0.5pt, arc=1.6mm,
  left=5pt, right=5pt, top=7pt, bottom=4pt, breakable=false,
  title={#2}, fonttitle=\sffamily\bfseries\scriptsize, coltitle=white,
  colbacktitle=cardslate,
  attach boxed title to top left={yshift=-2.4mm, xshift=3mm},
  boxed title style={arc=0.8mm, boxrule=0pt, left=5pt, right=5pt, top=1.6pt, bottom=1.6pt},
  #1}

This appendix reproduces the load-bearing prompts of both halves of the system
as prompt-card figures: the runtime guide prompt
(Figures~\ref{fig:prompt-protocol}--\ref{fig:prompt-rules2}, teal cards) and the
workflow-generation pipeline prompts
(Figures~\ref{fig:prompt-essentials}--\ref{fig:prompt-planreview}, slate cards).
Unicode punctuation is transliterated to ASCII, and the verifier's next-step
instruction is named \emph{remediation} consistently; otherwise the text is
verbatim.
In these prompt excerpts, ``turn'' names a verifier invocation and the hard-gate
language states the template's authorization contract.  The reported advisory
configuration invokes the verifier at user-turn boundaries, intercepts the first
unauthorized mutating call after a user message, and then permits an immediate
retry after corrective guidance (\S\ref{sec:guide}).

The guide's system message is assembled once per domain by substituting three
placeholders---\texttt{\{policy\_doc\}} (the raw policy), \texttt{\{graph\_doc\}}
(the composed graph rendered as a topology section plus per-node specs), and
\texttt{\{tools\_doc\}} (the mutating/read-only tool partition plus the domain's closed
value vocabularies)---into a fixed template, with the per-turn task instruction
appended at the end so the whole thing is one cached static prefix
(Appendix~\ref{app:cost:cache}). Figure~\ref{fig:prompt-protocol} shows the
protocol, per-turn trigger, and output contract;
Figures~\ref{fig:prompt-rules1} and~\ref{fig:prompt-rules2} the judging rules
referenced by the main text's verifier description (\S\ref{sec:guide}).
The generation pipeline's shared system prompt carries the schema contract (the
eight node types with required fields, edge semantics, id rules) plus the
authoring essentials of Figure~\ref{fig:prompt-essentials}; the plan and
adversarial-review stage prompts are in Figure~\ref{fig:prompt-planreview}. The
remaining stages (plan review, per-subflow generation, subflow path review, main
wiring) restate subsets of the same contract scoped to their output file.

\section{Example workflow graphs}
\label{app:workflow-example}

\tikzset{
  wfex/.style={font=\scriptsize, every node/.append style={align=center}},
  wfact/.style={draw=black!60, rounded corners=1pt, fill=blue!6, minimum height=7mm, inner sep=2pt, text width=18mm, align=center},
  wfusr/.style={wfact, fill=yellow!18},
  wftool/.style={wfact, fill=green!12},
  wfauth/.style={draw=black!60, trapezium, trapezium left angle=70, trapezium right angle=110, fill=red!12, minimum height=7mm, inner sep=1pt, text width=19mm, align=center},
  wfdec/.style={draw=black!60, diamond, aspect=2.0, fill=violet!10, inner sep=0.5pt, text width=11mm, align=center},
  wfterm/.style={draw=black!60, rounded corners=3mm, fill=gray!15, minimum height=6mm, inner sep=2pt, text width=16mm, align=center},
  wfsub/.style={draw=black!50, dashed, rounded corners=1pt, fill=orange!8, minimum height=7mm, inner sep=2pt, text width=19mm, align=center},
  wfghost/.style={draw=black!40, densely dotted, rounded corners=1pt, fill=white, minimum height=7mm, inner sep=2pt, text width=20mm, text=black!60, align=center},
  wfe/.style={-{Stealth[length=1.8mm]}, line width=0.5pt, draw=black!70},
  wfbr/.style={wfe, dashed},
  wflab/.style={font=\tiny, text=black!55, inner sep=1pt},
}

This appendix visualizes one generated-schema workflow per domain, composed flat
as the guide's cached prefix renders it. Subflow composition prefixes node ids
(\texttt{identify\_user.load\_profile}, \texttt{book\_flow.authorize\_book}), so
the guide addresses every node of every inlined subflow by a stable path-like
id. The three figures illustrate the intended schema: every depicted solid edge
carries the same label (\texttt{when:\,satisfied}), depicted mutating tools use
a \texttt{tool\_authorization} node (trapezoid) with stated prerequisites
upstream, and an authorize$\to$verify pair represents the post-call success
check. These visualizations do not establish complete mutating-tool coverage for
every frozen artifact; Appendix~\ref{app:workflow-validation} reports the
manual verification and programmatic checks.
Node colors match the main-text node vocabulary: \ntentry{}, \ntagent{},
\ntuser{}, \nttool{}, \ntauth{}, \ntdecision{}, and \ntsubflow{}.

\paragraph{Airline (Figure~\ref{fig:workflow-example}).} The shared main spine
(intake $\to$ identify $\to$ classify) and the full transactional
\texttt{book\_reservation} path: trip and passenger collection, a mandatory
search whose \texttt{TOOL\_RESULT} grounds the flight arguments, the insurance
disclosure, and the shared summary-and-confirm exchange, all upstream of the
gate. The classifier routes each reconciled request into its request-type subflow;
\texttt{general} and \texttt{transfer} are terminal branches, not subflows.

\paragraph{Retail (Figure~\ref{fig:workflow-retail}).} The main spine performs
intake and shared identification before the classifier dispatches each request
to cancellation, modification, return, exchange, account, or terminal handling.
The expanded \texttt{cancel\_pending\_order} branch is the canonical \emph{flat
gate conjunction} of Appendix~\ref{app:policy:retail}: locate the order, verify
its status is exactly \texttt{pending}, obtain the closed-set cancellation
reason, confirm, act, and verify.

\paragraph{Telecom (Figure~\ref{fig:workflow-telecom}).} After intake and shared
customer identification, the classifier dispatches requests to billing, line,
plan, roaming, and troubleshooting subflows. The expanded overdue-bill branch
invokes \texttt{pay\_overdue\_bill}, whose eligibility gates precede
confirmation and authorization of \texttt{send\_payment\_request}. The expanded
MMS branch invokes \texttt{troubleshoot\_mms}, an ordered traversal of service
and data prerequisites
followed by the documented MMS causes and a closing resolve-or-transfer
decision (Appendix~\ref{app:policy:telecom}).

\subsection{Turn-by-turn guide example}
\label{app:turn-example}

For the airline task \emph{``Book me on HAT136 JFK$\rightarrow$LAX, Nov 15,''}
the agent jumps straight to \texttt{book\_reservation}; the guide walks it back
through the policy path, one remediation step per turn. Turn~1 stops at
\texttt{identify\_user.ask\_user\_id} and asks for the user id; turn~2 directs
\texttt{get\_user\_details}; subsequent turns walk through trip collection,
flight search, rule validation, payment, baggage/insurance computation, and
summary-and-confirm. Only after the upstream nodes are satisfied does
\texttt{book\_flow.authorize\_book\_reservation} set
\texttt{authorize\_tool}; the following turn verifies the successful
\texttt{TOOL\_RESULT} and closes the request.

\begin{figure*}[t]
\begin{minipage}[t]{0.492\textwidth}
\vspace{0pt}%
\begin{guidecard}{Guide system prompt \,\textperiodcentered\, protocol}
\begin{lstlisting}[style=promptbody]
You are the policy guide for a customer-service agent. You do NOT talk to the user; you read the policy and the conversation and tell the agent what to do next by tracking where each of the user's requests sits in the workflow graph.

## How this works (read carefully)
We work through a single ongoing chat, one message per agent turn. I (the runtime) keep the authoritative state in code; each turn I append the new conversation and the current tracked state, and you return the COMPLETE updated state plus each blocked request's remediation. You have the entire workflow graph below, so you traverse it yourself in your own reasoning -- there is no per-node back-and-forth. Each turn you do two steps:
1. RECONCILE the open requests (intents): start a NEW request at the graph entry node, keep a CONTINUING request at its recorded node/status/authorization, DROP a request the user abandoned, and MERGE two entries that are the same request into one.
2. TRAVERSE each open intent from its current node: evaluate that node against its expectation and satisfying condition (ground truth = tool results); if satisfied, step to the next node and evaluate again; STOP at the FIRST unsatisfied node -- that becomes the intent's current node and you write its remediation; at a decision node follow the branch that matches the request; authorize a WRITE tool only when its policy prerequisites are all met; mark an intent that reaches a terminal node done.
Do the work as REASONING you write out step by step, and only AFTER the reasoning emit the final state as JSON. Reason first, commit second -- never write the JSON cold. Your final JSON is the memory you are guaranteed to carry forward, so every fact you will need next turn must live in it.
\end{lstlisting}
\end{guidecard}
\begin{guidecard}{Guide system prompt \,\textperiodcentered\, per-turn trigger (cached)}
\begin{lstlisting}[style=promptbody]
Update the state machine for this agent turn. FIRST reason in plain text -- reconcile the open requests, then traverse the graph for each open intent node-by-node, and for every node quote its satisfying criterion, cite the evidence, and decide SATISFIED / NOT SATISFIED (stop at the first unsatisfied node). THEN, after the reasoning, emit the final state as the single fenced json block as the last thing in your message.
\end{lstlisting}
\end{guidecard}
\end{minipage}\hfill
\begin{minipage}[t]{0.492\textwidth}
\vspace{0pt}%
\begin{guidecard}{Guide system prompt \,\textperiodcentered\, output contract (JSON)}
\begin{lstlisting}[style=promptbody]
{
 "reconcile": {
  "reasoning": "<one line: which requests are open now, and what you added / closed / merged this turn>",
  "intents": [{"id": "<short stable id>", "request": "<one-line description with the concrete target>", "intent": "<classifier branch label>"}]
 },
 "traverse": [
  {
   "id": "<intent id>",
   "plan": "<this intent's whole arc in one line: the end outcome or WRITE it drives toward, the sub-steps that get there, and which are already done>",
   "walk": [
    {"node": "<node id>", "reasoning": "<the SATISFIED/NOT-SATISFIED judgement>", "satisfied": true},
    {"node": "<next node id>", "reasoning": "<...>", "satisfied": false}
   ],
   "node": "<the node you stopped at>",
   "status": "open | done",
   "authorize_tool": "<WRITE tool to open at its authorization node, or null>",
   "selection": {
    "target": "<the specific record(s) this intent acts on>",
    "ruled_out": ["<each candidate examined and rejected, with the reason>"],
    "grounded_values": {"<arg>": {"value": "<copied verbatim from the result that established it>", "source": "<which tool result + record, or 'user message'>"}}},
   "remediation": "<the exact next action for the agent if blocked; empty if done>"
  }
 ],
 "transfer": false,
 "summary": "<compact running recap of the working context the structured fields do not already capture>"
}
\end{lstlisting}
\end{guidecard}
\end{minipage}
\caption{Guide system prompt, part 1 of 3. \textbf{Left:} the protocol block
(reconcile-then-traverse, one generation per fired turn) and the per-turn
trigger, which is folded into the cached static prefix rather than re-sent.
\textbf{Right:} the Part-2 output contract---the single fenced JSON block the
runtime parses into code-owned state.}
\label{fig:prompt-protocol}
\end{figure*}

\begin{figure*}[t]
\begin{minipage}[t]{0.492\textwidth}
\vspace{0pt}%
\begin{guidecard}{Judging rules \,\textperiodcentered\, ground truth}
\begin{lstlisting}[style=promptbody]
- Ground truth = tool results. A fact, eligibility condition, or completed action counts as established ONLY when a TOOL_RESULT in the conversation confirms it (directly or by your reasoning over tool results) -- NOT because the user asserted it, told you to assume it, or stated it as a given, and NOT because the agent merely said it; a value the agent computes from already-established inputs is itself established -- but a decision or argument that turns on a numeric or temporal computation must be worked out step by step in your reasoning and taken from those steps, never asserted as a conclusion; the user's own choices, consent, and preferences are established by the user's message, but a factual or eligibility condition the policy gates on is never established by the user's word -- when the user supplies or assumes one, delegate the read-only tool that verifies it and judge the condition from that result before relying on it, and a tool result that contradicts the user's claim governs.
\end{lstlisting}
\end{guidecard}
\end{minipage}\hfill
\begin{minipage}[t]{0.492\textwidth}
\vspace{0pt}%
\begin{guidecard}{Judging rules \,\textperiodcentered\, authorization}
\begin{lstlisting}[style=promptbody]
- Authorization. A WRITE (mutating) tool may be authorized ONLY when every policy prerequisite for that specific tool is met from tool-confirmed facts (plus the user's own consent where the policy asks for it) -- apply exactly the prerequisites the policy states for that tool, adding none it does not state, so once all of them are met you authorize rather than withholding for a condition you inferred. When you authorize, the runtime opens that tool's gate so the agent can call it; until then the runtime hard-blocks the call. Whenever your remediation instructs the agent to call a WRITE tool you have judged its prerequisites met, so set that intent's authorize_tool to that tool the same turn -- never instruct a WRITE call while leaving authorize_tool null. The authorizing remediation must state the exact arguments the agent must pass -- every id and value from grounded results, with any amounts, counts, or derived figures computed from the state the requested changes produce rather than from the prior state, so they reconcile with the tool's requirements -- and cover every change the user requested, so the agent does not guess, miscompute, or omit a step. To change specific fields of an existing record, reuse the record's current values for the fields the user is not changing rather than searching for or re-collecting new ones; and never withhold authorization to first establish a value the write tool itself computes or returns -- such an output is not a prerequisite. When the tool applies to several records, draw each call's arguments from that record's own grounded data and confirm the pairing before authorizing -- a value belonging to one record must never cross into another, since the call cannot be undone.
\end{lstlisting}
\end{guidecard}
\end{minipage}
\caption{Guide system prompt, part 2 of 3: judging rules (i)---evidence
grounding and mutating-tool authorization.}
\label{fig:prompt-rules1}
\end{figure*}

\begin{figure*}[t]
\begin{minipage}[t]{0.492\textwidth}
\vspace{0pt}%
\begin{guidecard}{Judging rules \,\textperiodcentered\, sourcing, transfer, fixes, reads}
\begin{lstlisting}[style=promptbody]
- Source every write argument. Before authorizing a WRITE, record in grounded_values where each argument's value came from -- the tool result and record it was copied from, or 'user message'. A value whose source is the user's word for a field a record owns is not grounded: read it from that record and use the record's value.
- Transfer. Signal transfer only when a request cannot be handled within policy at all (e.g. an action the policy reserves for a human, or the user insisting on a policy-violating action). Transfer ends the whole conversation, so signal it only when no open request can still be advanced within policy; when one request is blocked but others remain handleable, refuse only the blocked one and keep completing the rest rather than transferring. Never transfer an action the policy actually permits. transfer_to_human_agents is that signal, not an ordinary tool to authorize: whenever your remediation instructs the agent to call it you have judged the whole remaining task unhandleable, so set the top-level transfer to true the same turn -- the runtime opens that call only when transfer is true, so a remediation to transfer while transfer stays false contradicts itself and is blocked.
- Carry out the fix, not just name it. When resolving an open request requires a corrective action the graph gates inside another intent's subflow, open that action as an additional active intent and traverse its subflow so its tool can be authorized; the original request is resolved only once every corrective action its situation requires has been carried out, not when the cause is merely identified.
- Delegate READ tools to obtain facts. Establish any detail a READ/lookup tool can supply, or that a prior tool result already holds, from that result -- delegate the lookup or read the loaded data -- rather than asking the user. When advancing a request needs an identifier, record, or argument value the user has not supplied, do not ask for it -- direct the agent to enumerate the candidate records the READ tools return and select every one whose contents match the request's described attributes or the criterion the user stated, deriving each argument from that loaded data; a request that describes its target by attributes rather than by id is satisfied only once every matching record has been handled, not after the first.
\end{lstlisting}
\end{guidecard}
\end{minipage}\hfill
\begin{minipage}[t]{0.492\textwidth}
\vspace{0pt}%
\begin{guidecard}{Judging rules \,\textperiodcentered\, the remediation contract}
\begin{lstlisting}[style=promptbody]
- Remediation. A remediation is the exact next action for the agent: which tool to call with which arguments, or the value the user asked for computed from the tool results and stated back to them, or which detail to ask the user for (ask the user ONLY for things no tool can supply), or that the agent must refuse and the precise policy reason. A request for a value is resolved only once the agent has stated that value, not when a related action is done. When a node is not satisfied, the remediation must name the specific prerequisite that is missing and the concrete action that would obtain or satisfy it, never a generic statement that the conditions are unmet. Keep it concrete and grounded in the policy and the conversation. When the next steps along the intent's path are agent-side and already fully determined -- none needing a reply from the user and none whose arguments depend on a tool result you do not yet have -- write one remediation listing those ordered steps to carry out in a single stretch rather than one step per turn, splitting only at the first step that needs the user or a result you do not yet have. When the user has stated a selection or optimization criterion, first enumerate every candidate in the space the criterion ranges over from the tool results, then identify the single winning option by comparing that criterion across all of them and pinning the winning option's exact arguments -- never select from a partial candidate set, offer an unranked list, or ask the user to choose among candidates the criterion already decides. Every identifier you place in a remediation must appear verbatim in the specific tool result you are selecting it from -- never carry an identifier over from a different record (such as the one already on file) or introduce one not present in that result.
\end{lstlisting}
\end{guidecard}
\end{minipage}
\caption{Guide system prompt, part 3 of 3: judging rules (ii)---argument
sourcing, transfer scoping, corrective-fix traversal, read-tool delegation
(left) and the remediation contract (right).}
\label{fig:prompt-rules2}
\end{figure*}

\begin{figure*}[t]
\begin{minipage}[t]{0.492\textwidth}
\vspace{0pt}%
\begin{gencard}{Generation system prompt \,\textperiodcentered\, runtime contract}
\begin{lstlisting}[style=promptbody]
A proactive verifier reads the WHOLE graph plus the full conversation each turn, tracks where each open request sits, and writes the agent's next-step remediation. Two things the graph must make enforceable:
- Mutating tools are gated. A tool_authorization node is the choke point for one WRITE tool: the agent may call that tool ONLY after the guide authorizes it (its upstream prerequisites met) and the runtime confirms success at the following verify node. So every prerequisite/eligibility/confirmation a mutation requires must sit UPSTREAM of its authorization node.
- Faithfulness. The guide can only enforce what the graph encodes. The graph must reflect the policy and the domain's tool/task properties exactly.
\end{lstlisting}
\end{gencard}
\begin{gencard}{Generation system prompt \,\textperiodcentered\, authoring essentials 1--3}
\begin{lstlisting}[style=promptbody]
The bottom line: the graph must (A) reflect the policy faithfully, (B) reflect the domain's tool/task properties, and (C) be valid + minimal.
1. Cover every WRITE tool. Each mutating tool is reachable through some intent and ends in an authorize_<short> (tool_authorization) -> verify_<tool> (agent_action) pair -- the authorization choke point and the post-call success check (the tool can fail, so the verify node confirms a successful TOOL_RESULT and, on error, directs a correctable retry). Place everything the policy requires before a mutation UPSTREAM of its authorization. A value the mutating tool itself computes or returns is an output, not a prerequisite.
2. Gates are outcome-framed. An eligibility/permissibility gate is SATISFIED only if the permitting condition actually HOLDS -- never "has the agent checked X?" (that flips true the moment the check runs, even when it concludes ineligible). When the condition fails, the gate stays NOT_SATISFIED and its remediation REFUSES. Keep gates as agent_action nodes on the main path -- do not model refusal as a branch to a dead-end exit.
3. Ground facts in tools; quote the policy exactly. A fact/eligibility condition counts only when established by a tool result -- not a user assertion. The user's own consent/preference is established by the user's message. Where the environment gates on an exact status/enum literal, quote it and require exact equality (a qualified variant is a different value), and name whose field. Quote limits/prices/amounts/time-windows verbatim, identically across expectation/evaluation_prompt/remediation_template. Encode only conditions the policy states -- never invent a check no tool/data can establish.
\end{lstlisting}
\end{gencard}
\end{minipage}\hfill
\begin{minipage}[t]{0.492\textwidth}
\vspace{0pt}%
\begin{gencard}{Generation system prompt \,\textperiodcentered\, authoring essentials 4--7}
\begin{lstlisting}[style=promptbody]
4. Classifier completeness. The main classify decision has one branch per intent + general (anything unsupported -> a non-transfer refusal exit) + transfer (must go to a human). Operations the policy forbids outright get NO subflow -- they route to general.
5. Right tool side. tool_call/tool_authorization may name ONLY tools from the AGENT inventory. An action the END USER performs on their own device is an agent_action that instructs the user and judges their reported outcome -- never a tool node (which could never be satisfied).
6. Minimal + faithful. Author the fewest nodes that enforce the policy. Fold a validation/disclosure into the node that collects its data; merge related collection steps rather than one node per policy sentence. Every normative policy statement must be enforced by some node (or be genuinely out of graph scope). When the policy enumerates the possible causes of a problem, the flow that resolves it must check every documented cause; a cause whose remedy is a mutating tool must be its OWN checkpoint that reaches that tool's authorization.
7. Main spine + shared subflows. main: entry -> intake (greet + open question) -> identify (shared identification subflow) -> classify -> intent subflow anchors -> exits ["exit_normal","exit_general","exit_transfer"]. Any procedure shared by 2+ intents may be its own subflow. When the domain has user-owned records that requests target, the identification subflow ends with a tool_call that loads the authenticated user's full account record.
\end{lstlisting}
\end{gencard}
\end{minipage}
\caption{Workflow-generation system prompt: the runtime contract the graph must
satisfy, and the authoring-essentials contract every stage must follow.}
\label{fig:prompt-essentials}
\end{figure*}

\begin{figure*}[t]
\begin{minipage}[t]{0.492\textwidth}
\vspace{0pt}%
\begin{gencard}{Stage 2 \,\textperiodcentered\, plan requirements}
\begin{lstlisting}[style=promptbody]
1. Cover every WRITE tool. Each mutating tool is reachable through at least one intent subflow ending in its authorize -> verify pair.
2. Intents from the policy, not just the tools. Include advisory/no-tool intents the policy describes and procedure-driven intents from any troubleshooting manual. Unify all policy documents into ONE flat intent taxonomy.
3. Two archetypes. Transactional (mutate a record): load the target -> gate eligibility -> confirm -> execute. Troubleshooting (diagnose a symptom against a manual): an ordered checklist -- check a documented cause -> apply its fix -> re-test -> continue; resolve on a passing re-test, escalate after all causes exhausted. An agent WRITE-tool fix routes into that tool's mutating subflow; a user-device fix is an agent_action instructing the user.
4. Shared subflows. Plan the identification subflow and a present_summary_and_confirm subflow.
5. Per-subflow skeletons: ordered steps, id/type/tool/gist citing the policy line(s). Order: load target (tool_call) -> eligibility gate(s) -> collect details/disclosures -> summary_and_confirm -> authorize -> verify -> exit.
6. Coverage audit. Walk the policy section by section; assign every normative statement to a step, or list it as excluded (with a reason).
7. Lean. Fewest subflows and nodes that enforce the policy.
\end{lstlisting}
\end{gencard}
\end{minipage}\hfill
\begin{minipage}[t]{0.492\textwidth}
\vspace{0pt}%
\begin{gencard}{Stage 6 \,\textperiodcentered\, adversarial review checklist}
\begin{lstlisting}[style=promptbody]
1. Policy completeness. Walk the policy line by line. Every normative statement must be enforced by some node's criteria (or be genuinely out of graph scope). Fix any missing rule, wrong number, or weakened condition.
2. Cause coverage (enumerate explicitly). Wherever the policy enumerates the possible causes/conditions of a problem, list every one and name the node whose criterion handles it. If any cause has no node, add it.
3. Essentials compliance. Gates outcome-framed; environment-gated enums quoted with exact equality and whose-field named; facts grounded in tool results; every WRITE tool has its authorize -> verify pair with everything the policy requires upstream. Tool nodes name only agent-inventory tools.
4. Graph sanity. Entry reaches every node; every node reaches an exit; decision branch labels/targets consistent; classifier covers all intents + general + transfer.
5. Over-restriction sweep. A node whose criterion cannot be satisfied from the available tools/data is a transfer machine -- relax it to the policy's actual condition.

If nothing needs changing, output exactly:
NO_CHANGES
\end{lstlisting}
\end{gencard}
\end{minipage}
\caption{Workflow-generation stage prompts: the plan stage's requirements (left)
and the adversarial reviewer's checklist (right).}
\label{fig:prompt-planreview}
\end{figure*}

\begin{figure*}[t]
\centering
\small Figures~\ref{fig:workflow-example}--\ref{fig:workflow-telecom} use the
following shared node notation.\par\medskip
\resizebox{\textwidth}{!}{%
\begin{tikzpicture}[wfex]
\node[wfact, text width=15mm, minimum height=5mm] at (0.9,0) {{\tiny agent\_action}};
\node[wfusr, text width=15mm, minimum height=5mm] at (3.5,0) {{\tiny user\_input}};
\node[wftool, text width=15mm, minimum height=5mm] at (6.1,0) {{\tiny tool\_call (read)}};
\node[wfauth, text width=20mm, minimum height=5mm] at (8.8,0) {{\tiny tool\_authorization}\\[-2pt]{\tiny (authorize mutation)}};
\node[wfdec, text width=9mm] at (11.5,0) {{\tiny decision}};
\node[wfterm, text width=12mm, minimum height=5mm] at (13.7,0) {{\tiny entry/exit}};
\node[wfsub, text width=15mm, minimum height=5mm] at (15.9,0) {{\tiny subflow anchor}};
\end{tikzpicture}%
}
\end{figure*}

\begin{figure*}[t]
\centering
\resizebox{\textwidth}{!}{%
\begin{tikzpicture}[wfex]
\node[wfterm] (start) at (0,0) {\texttt{start}\\[-1pt]{\tiny entry}};
\node[wfact]  (intake) at (2.6,0) {\texttt{intake}\\[-1pt]{\tiny agent\_action}};
\node[wfusr]  (askid) at (5.4,0) {\texttt{ask\_user\_id}\\[-1pt]{\tiny user\_input}};
\node[wftool] (loadprof) at (8.2,0) {\texttt{load\_profile}\\[-1pt]{\tiny tool\_call:}\\[-2pt]{\tiny get\_user\_details}};
\node[wfdec]  (classify) at (11.2,0) {\texttt{classify}\\[-1pt]{\tiny decision}};
\node[wfterm] (exitgen) at (14.4,0.9) {\texttt{exit\_general}\\[-1pt]{\tiny exit}};
\node[wfterm] (exittrf) at (14.4,-0.9) {\texttt{exit\_transfer}\\[-1pt]{\tiny exit}};
\node[wfghost] (others) at (14.6,-2.6) {other request-type subflows\\(modify, cancel,\\refund, \ldots)};
\draw[wfe] (start) -- (intake);
\draw[wfe] (intake) -- (askid);
\draw[wfe] (askid) -- (loadprof);
\draw[wfe] (loadprof) -- (classify);
\draw[wfbr] (classify) -- node[wflab, above, sloped]{general} (exitgen);
\draw[wfbr] (classify) -- node[wflab, below, sloped]{transfer} (exittrf);
\draw[wfbr] (classify) -- node[wflab, below, sloped]{\ldots} (others);
\begin{scope}[on background layer]
\node[draw=black!45, dashed, rounded corners=2pt, fit={(askid) (loadprof)}, inner sep=2.5mm,
      label={[wflab, anchor=south west]north west:\texttt{identify\_user} subflow (shared)}] {};
\end{scope}
\node[wfact]  (trip) at (0.9,-2.6) {\texttt{collect\_trip}\\[-1pt]{\tiny agent\_action}};
\node[wftool] (search) at (3.5,-2.6) {\texttt{search\_flights}\\[-1pt]{\tiny tool\_call:}\\[-2pt]{\tiny search\_*\_flight}};
\node[wfact]  (pax) at (6.1,-2.6) {\texttt{collect\_}\\[-2pt]\texttt{passengers}\\[-1pt]{\tiny agent\_action}};
\node[wfact]  (pay) at (8.7,-2.6) {\texttt{collect\_payment}\\[-1pt]{\tiny agent\_action}};
\node[wfact]  (ins) at (11.3,-2.6) {\texttt{offer\_insurance}\\[-1pt]{\tiny agent\_action}};
\node[wfact]  (bags) at (11.3,-4.6) {\texttt{collect\_baggage}\\[-1pt]{\tiny agent\_action}};
\node[wfact]  (summ) at (8.7,-4.6) {\texttt{present\_summary}\\[-1pt]{\tiny agent\_action}};
\node[wfusr]  (conf) at (6.1,-4.6) {\texttt{explicit\_}\\[-2pt]\texttt{confirmation}\\[-1pt]{\tiny user\_input}};
\node[wfauth] (authb) at (3.5,-4.6) {\texttt{authorize\_book}\\[-1pt]{\tiny tool\_authorization:}\\[-2pt]{\tiny book\_reservation}};
\node[wfact]  (verif) at (0.9,-4.6) {\texttt{verify\_book\_}\\[-2pt]\texttt{reservation}\\[-1pt]{\tiny agent\_action}};
\node[wfterm] (booked) at (0.9,-6.3) {\texttt{booked}\\[-1pt]{\tiny exit}};
\draw[wfbr] (classify.south) to[out=-90,in=90] node[wflab, above, sloped]{book flight} (trip.north);
\draw[wfe] (trip) -- (search);
\draw[wfe] (search) -- (pax);
\draw[wfe] (pax) -- (pay);
\draw[wfe] (pay) -- (ins);
\draw[wfe] (ins.south) to[out=-90,in=90] (bags.north);
\draw[wfe] (bags) -- (summ);
\draw[wfe] (summ) -- (conf);
\draw[wfe] (conf) -- (authb);
\draw[wfe] (authb) -- (verif);
\draw[wfe] (verif) -- (booked);
\begin{scope}[on background layer]
\node[draw=black!45, dashed, rounded corners=2pt, fit={(summ) (conf)}, inner sep=2.5mm,
      label={[wflab, anchor=north west]south west:\texttt{present\_summary\_and\_confirm} subflow (shared)}] {};
\node[draw=black!30, rounded corners=3pt, fit={(trip) (ins) (bags) (verif) (booked)}, inner sep=4.5mm,
      label={[wflab, anchor=south east]north east:\texttt{book\_reservation} subflow (composed as \texttt{book\_flow.*})}] {};
\end{scope}
\end{tikzpicture}%
}
\caption{\textbf{Airline}: the composed workflow graph---the shared spine
(start, intake, the \texttt{identify\_user} subflow, classify)
and the full transactional \texttt{book\_reservation} request path. Solid edges
fire only \texttt{when:\,satisfied}; dashed edges are classifier branches. The
mutating tool is represented downstream of its \texttt{tool\_authorization}
node (trapezoid), immediately followed by the verify node that demands a
successful \texttt{TOOL\_RESULT}.}
\label{fig:workflow-example}
\end{figure*}

\begin{figure*}[t]
\centering
\resizebox{\textwidth}{!}{%
\begin{tikzpicture}[wfex]
\node[wfterm] (rstart) at (0,1.6) {\texttt{start}\\[-1pt]{\tiny entry}};
\node[wfact]  (rintake) at (2.4,1.6) {\texttt{intake}\\[-1pt]{\tiny agent\_action}};
\node[wfsub]  (rid) at (5.0,1.6) {\texttt{identify\_user}\\[-1pt]{\tiny shared subflow}};
\node[wfdec]  (rclass) at (7.7,1.6) {\texttt{classify}\\[-1pt]{\tiny decision}};
\node[wfact] (rgen) at (10.6,2.5) {\texttt{refuse\_general}\\[-1pt]{\tiny agent\_action}};
\node[wftool] (rtrf) at (10.6,0.7) {\texttt{transfer\_case}\\[-1pt]{\tiny tool\_call}};
\node[wfghost] (rothers) at (13.7,1.6) {other request subflows\\(modify, return,\\exchange, account)};
\draw[wfe] (rstart) -- (rintake);
\draw[wfe] (rintake) -- (rid);
\draw[wfe] (rid) -- (rclass);
\draw[wfbr] (rclass) -- node[wflab, above, sloped]{general} (rgen);
\draw[wfbr] (rclass) -- node[wflab, below, sloped]{transfer} (rtrf);
\draw[wfbr] (rclass) -- node[wflab, above]{other intents} (rothers);
\node[wfsub]  (locate) at (0,-1.4) {\texttt{locate}\\[-1pt]{\tiny subflow:}\\[-2pt]{\tiny locate\_order}};
\node[wfact]  (status) at (2.7,-1.4) {\texttt{verify\_status\_}\\[-2pt]\texttt{pending}\\[-1pt]{\tiny agent\_action (gate)}};
\node[wfact]  (reason) at (5.4,-1.4) {\texttt{obtain\_reason}\\[-1pt]{\tiny agent\_action}};
\node[wfsub]  (sac) at (8.1,-1.4) {\texttt{summary\_and\_}\\[-2pt]\texttt{confirm}\\[-1pt]{\tiny subflow (shared)}};
\node[wfauth] (auth) at (10.9,-1.4) {\texttt{authorize\_cancel}\\[-1pt]{\tiny tool\_authorization:}\\[-2pt]{\tiny cancel\_pending\_order}};
\node[wfact]  (verify) at (13.7,-1.4) {\texttt{verify\_cancel\_}\\[-2pt]\texttt{pending\_order}\\[-1pt]{\tiny agent\_action}};
\node[wfterm] (done) at (16.2,-1.4) {\texttt{cancelled}\\[-1pt]{\tiny exit}};
\draw[wfbr] (rclass.south) to[out=-90,in=90]
  node[wflab, above, sloped]{cancel pending order} (locate.north);
\draw[wfe] (locate) -- (status);
\draw[wfe] (status) -- (reason);
\draw[wfe] (reason) -- (sac);
\draw[wfe] (sac) -- (auth);
\draw[wfe] (auth) -- (verify);
\draw[wfe] (verify) -- (done);
\begin{scope}[on background layer]
\node[draw=black!30, rounded corners=3pt, fit={(locate) (done)}, inner sep=4mm,
      label={[wflab, anchor=south west]north west:\texttt{cancel\_pending\_order} subflow}] {};
\end{scope}
\end{tikzpicture}%
}
\caption{\textbf{Retail}: the shared entry--intake--identification spine and
classifier dispatch into request-specific subflows, with
\texttt{cancel\_pending\_order} expanded. Its status and reason gates, shared
summary-and-confirm step, authorization, and post-call verification form the
complete cancellation path. The gray box marks the expanded cancellation
subflow; the dotted box summarizes the other classifier branches.}
\label{fig:workflow-retail}
\end{figure*}

\begin{figure*}[t]
\centering
\resizebox{\textwidth}{!}{%
\begin{tikzpicture}[wfex]
\node[wfterm] (tstart) at (0,1.6) {\texttt{start}\\[-1pt]{\tiny entry}};
\node[wfact]  (tintake) at (2.4,1.6) {\texttt{intake}\\[-1pt]{\tiny agent\_action}};
\node[wfsub]  (tid) at (5.0,1.6) {\texttt{identify\_customer}\\[-1pt]{\tiny shared subflow}};
\node[wfdec]  (tclass) at (7.7,1.6) {\texttt{classify}\\[-1pt]{\tiny decision}};
\node[wfterm] (tgen) at (10.6,2.5) {\texttt{exit\_general}\\[-1pt]{\tiny exit}};
\node[wfsub] (ttrf) at (10.6,0.7) {\texttt{transfer\_to\_human}\\[-1pt]{\tiny subflow}};
\node[wfghost] (tothers) at (13.8,1.6) {other request subflows\\(line, plan, roaming,\\data, cellular)};
\draw[wfe] (tstart) -- (tintake);
\draw[wfe] (tintake) -- (tid);
\draw[wfe] (tid) -- (tclass);
\draw[wfbr] (tclass) -- node[wflab, above, sloped]{general} (tgen);
\draw[wfbr] (tclass) -- node[wflab, below, sloped]{transfer} (ttrf);
\draw[wfbr] (tclass) -- node[wflab, above]{other intents} (tothers);
\node[wfact]  (obill) at (0,-1.5) {\texttt{obtain\_bill\_id}\\[-1pt]{\tiny agent\_action}};
\node[wftool] (lbill) at (2.6,-1.5) {\texttt{load\_bill}\\[-1pt]{\tiny tool\_call:\,get\_bill}};
\node[wfact]  (sover) at (5.2,-1.5) {\texttt{verify\_status\_}\\[-2pt]\texttt{overdue}\\[-1pt]{\tiny agent\_action (gate)}};
\node[wfact]  (noawait) at (7.9,-1.5) {\texttt{verify\_no\_}\\[-2pt]\texttt{awaiting\_payment}\\[-1pt]{\tiny agent\_action (gate)}};
\node[wfsub]  (sac2) at (10.6,-1.5) {\texttt{summary\_and\_}\\[-2pt]\texttt{confirm}\\[-1pt]{\tiny subflow (shared)}};
\node[wfauth] (auth2) at (13.3,-1.5) {\texttt{authorize\_send\_}\\[-2pt]\texttt{payment\_request}\\[-1pt]{\tiny tool\_authorization}};
\node[wfact]  (ver2) at (16.0,-1.5) {\texttt{verify\_send\_}\\[-2pt]\texttt{payment\_request}\\[-1pt]{\tiny agent\_action}};
\node[wfterm] (sent) at (18.3,-1.5) {\texttt{sent}\\[-1pt]{\tiny exit}};
\draw[wfbr] (tclass.south) to[out=-90,in=90]
  node[wflab, above, sloped]{pay overdue bill} (obill.north);
\draw[wfe] (obill) -- (lbill);
\draw[wfe] (lbill) -- (sover);
\draw[wfe] (sover) -- (noawait);
\draw[wfe] (noawait) -- (sac2);
\draw[wfe] (sac2) -- (auth2);
\draw[wfe] (auth2) -- (ver2);
\draw[wfe] (ver2) -- (sent);
\begin{scope}[on background layer]
\node[draw=black!30, rounded corners=3pt, fit={(obill) (sent)}, inner sep=4mm,
      label={[wflab, anchor=south west]north west:\texttt{pay\_overdue\_bill} subflow (authorizes \texttt{send\_payment\_request})}] {};
\end{scope}
\node[wfact]  (cmms) at (0,-4.6) {\texttt{confirm\_mms\_}\\[-2pt]\texttt{issue}\\[-1pt]{\tiny agent\_action}};
\node[wfact]  (vsvc) at (2.6,-4.6) {\texttt{verify\_service}\\[-1pt]{\tiny agent\_action (prereq)}};
\node[wfact]  (vdata) at (5.2,-4.6) {\texttt{verify\_mobile\_}\\[-2pt]\texttt{data}\\[-1pt]{\tiny agent\_action (prereq)}};
\node[wfact]  (ntech) at (7.9,-4.6) {\texttt{check\_network\_}\\[-2pt]\texttt{tech}\\[-1pt]{\tiny agent\_action}};
\node[wfact]  (wific) at (10.6,-4.6) {\texttt{check\_wifi\_}\\[-2pt]\texttt{calling}\\[-1pt]{\tiny agent\_action}};
\node[wfact]  (perm) at (13.3,-4.6) {\texttt{check\_app\_}\\[-2pt]\texttt{permissions}\\[-1pt]{\tiny agent\_action}};
\node[wfact]  (apn) at (16.0,-4.6) {\texttt{check\_apn\_mmsc}\\[-1pt]{\tiny agent\_action}};
\node[wfdec]  (dout) at (16.0,-6.6) {\texttt{decide\_}\\[-2pt]\texttt{outcome}\\[-1pt]{\tiny decision}};
\node[wfterm] (resolved) at (13.3,-6.6) {\texttt{resolved}\\[-1pt]{\tiny exit}};
\node[wfact]  (tnode) at (16.0,-8.4) {\texttt{transfer\_node}\\[-1pt]{\tiny agent\_action}};
\node[wfterm] (trfd) at (13.3,-8.4) {\texttt{transferred}\\[-1pt]{\tiny exit}};
\draw[wfbr] (tclass.south) -- ++(0,-0.6)
  -| ([xshift=-7mm]cmms.west) -- (cmms.west);
\node[wflab, rotate=90] at (-0.65,-3.3) {MMS issue};
\draw[wfe] (cmms) -- (vsvc);
\draw[wfe] (vsvc) -- (vdata);
\draw[wfe] (vdata) -- (ntech);
\draw[wfe] (ntech) -- (wific);
\draw[wfe] (wific) -- (perm);
\draw[wfe] (perm) -- (apn);
\draw[wfe] (apn) -- (dout);
\draw[wfbr] (dout) -- node[wflab, above]{resolved} (resolved);
\draw[wfbr] (dout) -- node[wflab, right]{transfer} (tnode);
\draw[wfe] (tnode) -- (trfd);
\begin{scope}[on background layer]
\node[draw=black!30, rounded corners=3pt, fit={(cmms) (apn) (trfd)}, inner sep=4mm,
      label={[wflab, anchor=north west]south west:\texttt{troubleshoot\_mms} subflow (prerequisites Service\,$\subset$\,Data\,$\subset$\,MMS)}] {};
\end{scope}
\end{tikzpicture}%
}
\caption{\textbf{Telecom}: the shared entry--intake--identification spine and
classifier dispatch, with two request branches expanded. The overdue-bill branch
invokes \texttt{pay\_overdue\_bill}, which gates
\texttt{send\_payment\_request}; the MMS branch invokes
\texttt{troubleshoot\_mms}, which checks service and data prerequisites before
the documented MMS causes and closes by resolving or transferring. The dotted
box summarizes the remaining classifier branches.}
\label{fig:workflow-telecom}
\end{figure*}

\section{Trace validation and residual failures}
\label{app:trace-validation}

\subsection{Trace-rubric validation and denominators}
\label{app:trace-audit}
Telecom process-valid success uses all 160 task--trial pairs per condition,
including outcome failures. Step- and Trace-TCR condition on outcome success,
so their cohorts differ between systems. Table~\ref{tab:trace-audit-counts}
gives the diagnostic-variant counts: \pgs{} has the highest conditional
Trace-TCR; \pg{} completes the most process-valid tasks overall.

\begin{table}[t]
\centering
\scriptsize
\setlength{\tabcolsep}{3pt}
\begin{tabular}{lrrr}
\toprule
System & Outcome & Process-valid & Trace-TCR \\
\midrule
\direct{} & 79/160 & 28/160 & 35.4\% \\
\pgd{} & 88/160 & 21/160 & 23.9\% \\
\pgr{} & 104/160 & 68/160 & 65.4\% \\
\pgs{} & 78/160 & 63/160 & 80.8\% \\
\rowcolor{ourrowbg}
\pg{} & 142/160 & 90/160 & 63.4\% \\
\bottomrule
\end{tabular}
\caption{Saved Telecom trace-ledger counts (four trials on 40 test tasks).
Trace-TCR divides process-valid by outcome-successful pairs; process-valid
success uses all 160 pairs.}
\label{tab:trace-audit-counts}
\end{table}

Two authors independently constructed procedural rules and required event
orders in blinded reviews of the source policy and technical-support manual.
They compared their annotations, discussed discrepancies, and cross-validated
the final rules. Both validated all 15 flagged condition types and stratified
traces against the source policy, resolving conditional instructions through
policy context as well as event order.

\subsection{Residual outcome failures}
\label{app:residual-failures}
The 160 saved \pg{} Telecom test traces include 18 outcome failures. Ten end
with verifier-directed transfer despite an available policy-supported repair;
eight state that transfer is needed but omit
\texttt{transfer\_to\_human\_agents} before the next user-boundary firing,
combining actor omission with an intervention-boundary gap. These counts
describe terminal mechanisms rather than isolated causal effects or verifier
error rates. The 52 outcome-successful but process-invalid traces remain
unattributed.

\paragraph{Retry-gate stress testing.}
We tested the unchanged runtime gate at 13 saved CRAFT Airline checkpoints:
12 preconfirmation histories with authorization masks forced closed, and one
recorded interception with logged authorization states. Using saved histories
and call arguments, controlled actor/verifier outputs, and a deterministic
environment harness, we ran 92 scenarios (Table~\ref{tab:retry-gate-stress}),
including one replay of recorded correction and retry.

All initial closed writes were intercepted. Once the write gate disarmed,
forced same-turn retries, tool switches, and retries after corrective-verifier
exceptions were admitted. New user turns rearmed interception, and repeated
unauthorized transfers remained blocked.

\begin{table}[!htbp]
\centering
\scriptsize
\setlength{\tabcolsep}{2pt}
\begin{tabular}{p{0.43\columnwidth}rrrr}
\toprule
Scenario after interception & Cases & Calls & Admitted & Blocked \\
\midrule
Force one closed-write retry & 13 & 13 & 13 & 0 \\
Force three closed-write retries & 13 & 39 & 39 & 0 \\
Force ten closed-write retries & 13 & 130 & 130 & 0 \\
Switch to another closed write tool & 13 & 13 & 13 & 0 \\
Corrective-verifier exception, then retry & 13 & 13 & 13 & 0 \\
Repeat unauthorized transfer three times & 13 & 39 & 0 & 39 \\
First closed write in three new user-turn regions & 13 & 39 & 0 & 39 \\
Replay recorded authorization refresh and retry & 1 & 1 & 1 & 0 \\
\bottomrule
\end{tabular}
\caption{Trajectory-based fault-injection tests of the advisory gate. ``Cases''
are reused checkpoints; ``Calls'' are subsequent attempts after initial
interception. A closed tool is absent from the authorization mask; admission
denotes dispatch to the deterministic environment harness.}
\label{tab:retry-gate-stress}
\end{table}

\paragraph{Recorded recovery.}
In the recorded certificate trajectory, corrective verification after
interception added the certificate tool to the authorization set. The actor
then retried and the call executed, illustrating recovery through authorization
refresh.

\section{The Use of LLMs}
\label{app:llm-use}

We used LLMs solely for light editing, such as correcting grammatical errors
and polishing wording. They did not contribute to research ideation,
experiments, analysis, or substantive writing.

\end{document}